\documentclass{article}
\usepackage{iclr2027_conference,times}

\usepackage[T1]{fontenc}
\usepackage{amsmath,amssymb,amsthm}
\usepackage{bm}
\usepackage{booktabs}
\usepackage{float}
\usepackage{graphicx}
\usepackage{multirow}
\usepackage{nicefrac}
\usepackage{wrapfig}
\usepackage{xcolor}
\usepackage{hyperref}
\usepackage{url}
\usepackage{cleveref}
\usepackage[ruled,vlined]{algorithm2e}
\newenvironment{algorithm2e}[1][htbp]{\begin{algorithm}[#1]}{\end{algorithm}}

\hypersetup{hypertexnames=false,hidelinks}

\newtheorem{proposition}{Proposition}[section]
\newtheorem{lemma}{Lemma}[section]
\newtheorem{theorem}{Theorem}[section]

\title{HyperTransfer: Understanding the Equivalence  between Base Optimizer and Hyperball}

\iclrfinalcopy
\author{
Jinghui Yuan\thanks{Equal contribution.}\;\;\thanks{Work done while interning at StepFun.} \\
StepFun \\
\texttt{yuanjinghui001@gmail.com}
\And
Hongtao Zhang\footnotemark[1]\;\;\footnotemark[2] \\
University of Chinese Academy of Sciences, StepFun \\
\texttt{zhanghongtao24@mails.ucas.ac.cn}
\And
Jade Zou \\
Independent Researcher \\
\texttt{jade.zou@gmail.com}
\And
Tianyu Li\footnotemark[2] \\
Tsinghua University, StepFun \\
\texttt{tianyu-l22@mails.tsinghua.edu.cn}
\AND
Wenjie Zhou \\
University of Chinese Academy of Sciences \\
\texttt{zj4323005@gmail.com}
\And
Tianyu He\thanks{Corresponding authors.} \\
StepFun \\
\texttt{tianyu\_he@outlook.com}
\And
Wei Chen\footnotemark[3] \\
University of Chinese Academy of Sciences, Institute of Computing Technology,\\
Chinese Academy of Sciences \\
\texttt{chenwei2022@ict.ac.cn}
}

\begin{document}

\maketitle
\fancyhead{}
\fancyfoot[C]{\thepage}

\begin{abstract}
Hyperball optimizers constrain parameter norms and update only their directions, establishing a distinct paradigm for neural network optimization. 
Although this geometry appears fundamentally different from that of conventional Base Optimizers, which update both parameter norms and directions, we show that the two paradigms are dynamically equivalent for scale-invariant networks. 
Building on this equivalence, we propose HyperTransfer, which constructs a Hyperball optimizer that reproduces the dynamics of a target Base Optimizer using only its initialization and learning-rate schedule, without running the target optimizer itself. 
We further derive the inverse mapping and extend the framework to non-scale-invariant networks.
Experiments show that both HyperTransfer and the inverse mapping produce loss trajectories nearly identical to those of their targets, suggesting that Hyperball dynamics are governed primarily by the induced effective learning-rate schedule and optimizer state. 
Finally, we observe that matching nominal learning-rate decay ratios can yield different effective learning-rate decay ratios in Base and Hyperball optimizers, highlighting the need to account for this discrepancy when designing fair optimizer comparisons.


\end{abstract}

\section{Introduction}

As large language models (LLMs) continue to grow in scale and complexity,
training them stably yet efficiently remains a central challenge \citep{wen2025fantasticpretrainingoptimizers, li2026omniopttaxonomygeometrybenchmarking}. 
Hyperball optimizers \citep{wen2026fantasticpretrainingoptimizersii} have recently emerged as a promising answer on both fronts, achieving lower loss and more stable training without explicit weight decay \citep{ren2026rethinkinglanguagemodelscaling}. Their design is motivated by a structural property of modern LLMs: many weight matrices are immediately followed by normalization layers \citep{ba2016layernormalization} and are therefore exactly or approximately scale-invariant---rescaling them leaves the loss unchanged. For such parameters, only the direction matters, and the norm is a functionally redundant degree of freedom.

This motivation, however, leads to a puzzle. If the radial degree of freedom is redundant, removing it should be inconsequential. Yet it visibly changes how training unfolds. Starting from the same initialization under otherwise identical settings, a Base Optimizer such as Muon \citep{jordan2024muon} or Adam \citep{kingma2017adammethodstochasticoptimization}, which updates parameters without norm constraints, often decreases the loss substantially faster early in training, whereas its Hyperball counterpart closes the gap and attains a lower loss only in the final stage \citep{modded_nanogpt_2024}. 

One might suspect that this discrepancy stems from the imperfect scale invariance of practical architectures. To rule this out, we construct a strictly scale-invariant network, denoted by \(\mathcal{M}_1\); as shown in Figure~\ref{fig1}, the same crossover persists. This observation motivates the central question of this paper: \noindent\textbf{What accounts for the discrepancy in optimization dynamics
between a Base Optimizer and Hyperball?}

The answer, it turns out, lies within scale invariance itself. Through a simple derivation in \Cref{sec:hypertransfer}, We show that aligning the initialization, the effective learning rates, and the update directions of the weight matrices suffices to align the entire learning trajectories of a Base Optimizer and its Hyperball variant. Building on this understanding, we propose HyperTransfer and its inverse, which map the dynamics of a Base Optimizer onto its Hyperball variant and vice versa. We further extend both mappings beyond the scale-invariant setting through HyperTransfer+ and its inverse, which accommodate networks where strict scale invariance no longer holds.

In \Cref{c1,sec:hypertransfer-plus-experiments}, we empirically validate this framework on two network variants. The strictly scale-invariant network \(\mathcal{M}_1\) is obtained by modifying a GPT-2-based architecture \citep{radford2019language}. On \(\mathcal{M}_1\), HyperTransfer and Inverse HyperTransfer reproduce the target loss trajectories and recover the corresponding parameter-norm evolution and effective learning-rate schedules \citep{wan2020sphericalmotiondynamicslearning}. On the non-scale-invariant network \(\mathcal{M}_2\), HyperTransfer+ and Inverse HyperTransfer+ closely match the target trajectories, showing that accurate dynamical transfer extends beyond strict scale invariance.

Our framework also exposes a practical confound: a nominal schedule ending at \(0.1\times\) the peak learning rate yields an average effective final-to-peak ratio of only \(0.047\times\) for the Base Optimizer, compared with \(0.1\times\) for Hyperball. Thus, nominal schedule matching does not ensure effective-rate matching. \Cref{sec_nominal_vs_elr} further shows that MuonH remains superior after ratio matching, while Inverse HyperTransfer+ transfers this advantage to Muon through an induced learning-rate schedule.

\section{Notations and Related Work}
\label{sec:notations}

\subsection{Notation}

Let \(f:\mathbb{R}^{a\times b}\rightarrow\mathbb{R}\) denote the loss as a function of the parameter matrix \(w\in\mathbb{R}^{a\times b}\). Throughout this work, \(\|\cdot\|\) and \(\langle\cdot,\cdot\rangle\) denote the Frobenius norm and inner product, and \(t\) indexes training iterations.

At iteration \(t\), the Base Optimizer maintains the parameter \(w_t\), and we write \(\eta_t\), \(g_t\), \(m_t\), and \(u_t\) for its learning rate, gradient, momentum state, and update direction. The corresponding Hyperball quantities carry a superscript \(H\): \(w_t^H\), \(\eta_t^H\), \(g_t^H\), \(m_t^H\), and \(u_t^H\). Both optimizers share the momentum coefficient \(\beta\), and we use \(\lambda\) to denote the weight-decay coefficient. Hyperball Optimizer fixes the parameter norm to the radius \(R=\|w_0^H\|\).

Two derived quantities play a central role in our analysis. The \emph{effective learning rate} of the Base Optimizer, $\eta_t^{\mathrm{eff}} = \nicefrac{\eta_t\|u_t\|}{(1-\eta_t\lambda)\|w_t\|}$, measures the magnitude of its directional update relative to the parameter norm, after factoring out radial weight decay. The \emph{proxy norm} \(s_t\) is an online estimate of \(\|w_t\|\), the norm of the corresponding Base Optimizer parameter. 

Finally, \(\mathcal{M}_1\) denotes the strictly scale-invariant network, in which \(f(cw)=f(w)\) holds for every hidden weight matrix \(w\) and every \(c>0\), and \(\mathcal{M}_2\) denotes the non-scale-invariant network, for which this identity does not hold in general.

\subsection{Related Work}

\paragraph{Scale Invariance, Normalization, and Effective Learning Rate.}
Normalization makes many parameterizations insensitive to the scale of their weights, and a substantial body of work connects this property to the roles of weight decay and the learning rate. Studies ranging from the relationship between \(L_2\) regularization and normalization \citep{vanlaarhoven2017l2regularizationversusbatch} and the intrinsic learning rate \citep{li2020reconcilingmoderndeeplearning} to spherical views of normalized networks \citep{roburin2022sphericalperspectivelearningnormalization,kodryan2023trainingscaleinvariantneuralnetworks} and the analysis of rotational equilibrium \citep{kosson2024rotationalequilibriumweightdecay} show that parameter norms can affect optimization even when they do not affect the represented function. Subsequent work examines this phenomenon specifically for scale-invariant weights: AdamP \citep{heo2021adampslowingslowdownmomentum} shows that momentum can cause excessive norm growth and a premature decay of effective step sizes, \citet{xiang2019understandingdisharmonyweightnormalization} demonstrate that under weight normalization, weight decay acts by modulating the effective learning rate, and \citet{dangelo2024needweightdecaymodern} argue that weight decay shapes optimization dynamics rather than serving only as an explicit regularizer. These studies analyze how scale influences an  optimizer; in contrast, our work establishes an explicit, step-wise correspondence between Base and Hyperball optimization, including the transformation of learning rates and optimizer states.

\paragraph{Hyperspherical and Norm-Constrained Optimization.}
A separate line of work makes parameter norms explicit through reparameterization or constrained optimization. Weight Normalization \citep{salimans2016weightnormalizationsimplereparameterization} decouples the magnitude and direction of weight vectors, and more recent approaches, including nGPT \citep{loshchilov2025ngptnormalizedtransformerrepresentation} and Spectral Sphere Optimization \citep{xie2026controlledllmtrainingspectral}, constrain representations or weight matrices to hyperspherical spaces. Hyperball \citep{wen2026fantasticpretrainingoptimizersii} imposes fixed Frobenius-norm constraints on weight matrices and optimizer updates, while HyperP \citep{ren2026rethinkinglanguagemodelscaling} studies the scaling and learning-rate-transfer properties of this geometry. Constrained Parameter Regularization \citep{franke2024improvingdeeplearningoptimization} imposes adaptive upper bounds on the norms of individual parameter groups, and the approximately normalized Transformer \citep{franke2025learning} constrains parameter norms and representations to obtain a compact training space. These methods primarily employ norm constraints to improve stability, convergence, or scaling; our work takes a complementary perspective, treating Hyperball as a controlled setting for understanding why removing a functionally redundant radial degree of freedom can nevertheless change the optimization trajectory.

During the preparation of this manuscript, we became aware of three concurrent works \citep{9k,hy,liu2026effectivelearningrategoverns} that develop closely related ideas.

\section{HyperTransfer}
\label{sec:hypertransfer}

This section develops HyperTransfer on the scale-invariant network \(\mathcal M_1\). In \Cref{thg}, we derive the step-wise angular-alignment condition and the proxy norm recursion that yield exact trajectory matching between Base and Hyperball optimization, and we also provide the geometric interpretation of the construction. \Cref{app:inverse-hypertransfer} then derives the inverse mapping, and \Cref{c1} presents the  HyperTransfer and Inverse HyperTransfer algorithms together with empirical validation on \(\mathcal M_1\).

\subsection{Theory and Geometric Interpretation}
\label{thg}
For scale-invariant networks, the loss depends only on the parameter direction causing that if $ w_t/\|w_t\| = w_t^H/\|w_t^H\|$ for all $t$, then $ f(w_t) = f(w_t^H)$. In this sense, matching the direction of parameters at each iteration is sufficient to align the loss trajectories. Proposition ~\ref{prop1} gives a sufficient condition for maintain this alignment from one iteration to the next.
 
\begin{proposition}[Angular Update Alignment]
  \label{prop1}
  Assume that \(f\) is scale-invariant and \(w_0=w_0^H\). If, at each iteration \(t\),
  the angular-update condition is satisfied:
  \begin{equation}
          \eta_t^H \frac{u_t^H}{\|u_t^H\|}
    =
    \eta_t^{\mathrm{eff}} \frac{u_t}{\|u_t\|},
  \end{equation}
where $\eta_t^{\mathrm{eff}}=\frac{\eta_t\|u_t\|}{(1-\eta_t\lambda)\|w_t\|}$. Then we have \( {w_{t+1}} / \|w_{t+1}\|={w_{t+1}^H}/{\|w_{t+1}^H\|}\), and consequently \(f(w_{t+1})=f(w_{t+1}^H)\) for each iteration.
\end{proposition}

Proposition~1 reduces trajectory matching to the problem of matching the effective angular update at each iteration. In practice, however, the quantity \(\|w_t\|\) appearing in \(\eta_t^{\mathrm{eff}}\) is not available during a Hyperball run. To make the condition computable \emph{online}, we introduce a proxy norm \(s_t\) to track the corresponding Base parameter norm and derive its recursive update from the Base Optimizer dynamics. This leads to a executable HyperTransfer procedure, which is formalized in Theorem~\ref{thm:hypertransfer}.

\begin{lemma}
\label{lemma_gradient}
For a scale-invariant function $f(w)$ satisfying
\(
f(cw)=f(w),\  \forall c>0,
\)
its gradient satisfies:
\(\nabla f(cw)=\frac{1}{c}\nabla f(w).\)
\end{lemma}

\begin{theorem}[Equivalence of Base and Hyperball Optimization]
\label{thm:hypertransfer}
Assume that \(f\) is scale-invariant, \(w_0=w_0^H\), and that Base Optimizer and Hyperball use the same sequence of minibatches. Denote \(s_t\) as the proxy norm that serves as an online scalar estimate of \(\|w_t\|\), initialized with \(s_0=R=\|w_0^H\|\). Whenever \(w_t=(s_t/R)w_t^H\), Lemma~\ref{lemma_gradient} gives \(g_t = \widetilde g_t^H=(R/s_t)g_t^H\). If the Hyperball momentum state is updated according to:
\begin{equation}
\label{momentum_align}
    m_t^H=\beta m_{t-1}^H+(1-\beta)\widetilde g_t^H,
\end{equation}
and the Hyperball learning rate is chosen as:
\begin{equation}
\label{elr_aglin}
    \eta_t^H=\frac{\eta_t\|u_t^H\|}{(1-\eta_t\lambda)s_t},
\end{equation}
where the update of \(s_t\) is given by:
\begin{equation}
\label{up}
   s_{t+1}^2
=(1-\eta_t\lambda)^2s_t^2
+\eta_t^2\|u_t^H\|^2
-2\eta_t(1-\eta_t\lambda)\frac{s_t}{R}
\langle u_t^H,w_t^H\rangle 
\end{equation}
then the angular-update condition in Proposition~\ref{prop1} holds at each iteration. Consequently, the parameters remain aligned throughout training, and
\(
f(w_{t+1})=f(w_{t+1}^H)\)
for all \(t\).

\end{theorem}
Theorem~\ref{thm:hypertransfer} shows how this one-step condition can be satisfied \emph{online}. Under the relation \(w_t=(s_t/R)w_t^H\), Lemma~\ref{lemma_gradient} implies that the rescaled Hyperball proxy gradient \(\widetilde{g}_t^H\) coincides with the Base gradient. Consequently,  the momentum recursion in Equation \eqref{momentum_align} synchronizes the optimizer states and yields the same update direction. The learning-rate choice in Equation \eqref{elr_aglin}, together with the proxy norm \(s_t\), then matches the effective angular displacement required by Proposition~\ref{prop1}. Moreover, the recursion for \(s_{t+1}\) computes the norm of the Base iterate solely from quantities available to the Hyperball optimizer, and therefore supplies the radial scale required at the next iteration.

Since \(w_0=w_0^H\) and \(s_0=R\), the alignment relation holds at iteration \(t=0\). Because of scaling invariance, the updates in Theorem~\ref{thm:hypertransfer} satisfy the angular-update condition of Proposition~\ref{prop1}, which preserves the relation at iteration \(t+1\). By induction, the normalized Base and Hyperball iterates remain aligned throughout training, yielding identical loss trajectories. Algorithm \ref{alg:hypertransfer} presents the algorithmic framework of HyperTransfer.

We further introduce Inverse HyperTransfer, which approximates the Hyperball dynamics using the dynamics of the Base Optimizer. Details are provided in Section \ref{app:inverse-hypertransfer}.

\begin{wrapfigure}[20]{r}{0.3\textwidth}
    \centering
    \includegraphics[width=\linewidth,page=1]{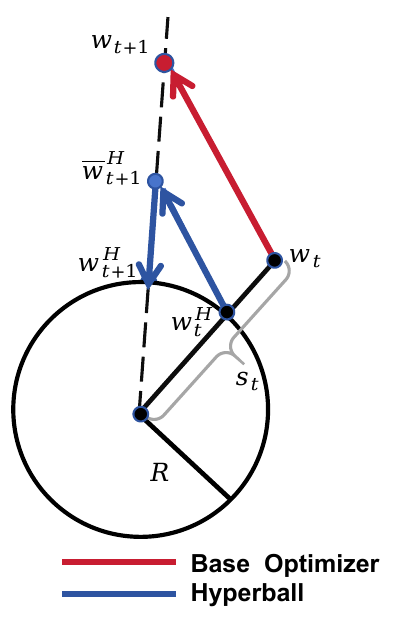}
    \caption{Geometric interpretation of HyperTransfer.}
    \label{geometric-interptation}
\end{wrapfigure}

\paragraph{Geometric Interpretation.}
Figure~\ref{geometric-interptation} illustrates the inductive mechanism underlying Proposition~\ref{prop1} and 
Let \(\bar{w}_{t+1}^H\) denote the pre-normalization Hyperball iterate shown in the figure,i.e. $ \bar{w}_{t+1}^H = w_t^H - \eta_t^H\frac{u_t}{\|u_t\|}$. Since scaling invariance, at iteration \(t\), suppose that the Base and Hyperball parameters satisfy \(w_t=(s_t/R)w_t^H\), where \(s_t=\|w_t\|\) and \(\|w_t^H\|=R\). Thus, \(w_t\) and \(w_t^H\) lie on the same positive ray from the origin and have identical directions.
Proposition~\ref{prop1} characterizes the one-step condition required to preserve this radial relation. After accounting for the different radial scales of the two iterates, the Base and Hyperball updates must induce the same effective angular displacement. Under this condition, the Base iterate \(w_{t+1}\) and the pre-normalization Hyperball iterate \(\bar{w}_{t+1}^H\) lie on the same positive ray, as illustrated by the dashed line in Figure~\ref{geometric-interptation}. Hyperball then renormalizes \(\bar{w}_{t+1}^H\) onto the sphere of radius \(R\), yielding \(w_{t+1}^H\). Since radial renormalization changes only the norm and preserves the direction, the normalized Base and Hyperball parameters remain aligned at iteration \(t+1\). The proxy norm \(s_t\) records the radial scale discarded by the Hyperball
projection. It is updated as the norm of the corresponding Base Optimizer
iterate.
Therefore, \(s_{t+1}=\|w_{t+1}\|\) and
\(w_{t+1}=(s_{t+1}/R)w_{t+1}^H\), restoring the radial relation required at
the next iteration.

\par
\WFclear

\begin{algorithm2e}[H]
\scriptsize
\caption{HyperTransfer}
\label{alg:hypertransfer}
\SetAlgoLined
\KwIn{
Loss function $f$; initial parameter $w_0^H$;
reference learning-rate schedule from  Base Optimizer $\{\eta_t\}_{t=0}^{T-1}$;
momentum coefficient $\beta$; weight decay coefficient $\lambda$;
number of iterations $T$
}
\KwOut{Updated parameter $w_T^H$, proxy norm $s_T$}
Initialize \(w_0^H\), \(R\leftarrow\|w_0^H\|\), and \(s_0\leftarrow R\)\;

\For{$t = 0,1,\dots,T-1$}{
Compute the gradient: $g_t^H=\nabla f(w_t^H)$;

Set the transferred gradient:
\(\widetilde g_t^H\leftarrow \frac{R}{s_t}g_t^H\)\;

Update momentum: \(m_t^H\leftarrow \widetilde g_t^H\) if \(t=0\); otherwise \(m_t^H\leftarrow \beta m_{t-1}^H+(1-\beta)\widetilde g_t^H\)\;

Compute the update: $u_{t}^H=\text{Newton-Schulz}(m_{t}^H)$;

Compute Hyperball learning rate: \(\eta_{t}^H \leftarrow \eta_t\frac{\|u_t^H\|}{(1-\eta_t\lambda)s_t}\)\;

Update parameter using Hyperball:
\(w_{t+1}^H \leftarrow
R\,(w_t^H-\eta_t^H R\,\frac{u_t^H}{\|u_t^H\|})
/\|w_t^H-\eta_t^H R\,\frac{u_t^H}{\|u_t^H\|}\|\)\;

Update proxy norm:
$s_{t+1}^2 \leftarrow
(1-\eta_t \lambda)^2 s_t^2
+ \eta_t^2 \|u_t^H\|^2
- 2\eta_t(1-\eta_t \lambda)
\frac{s_t}{R}\langle u_t^H, w_t^H \rangle$\;
}
\Return $w_T^H$, $s_T$;
\end{algorithm2e}

\begin{figure}[!t]
    \centering
    \includegraphics[width=0.9\linewidth]{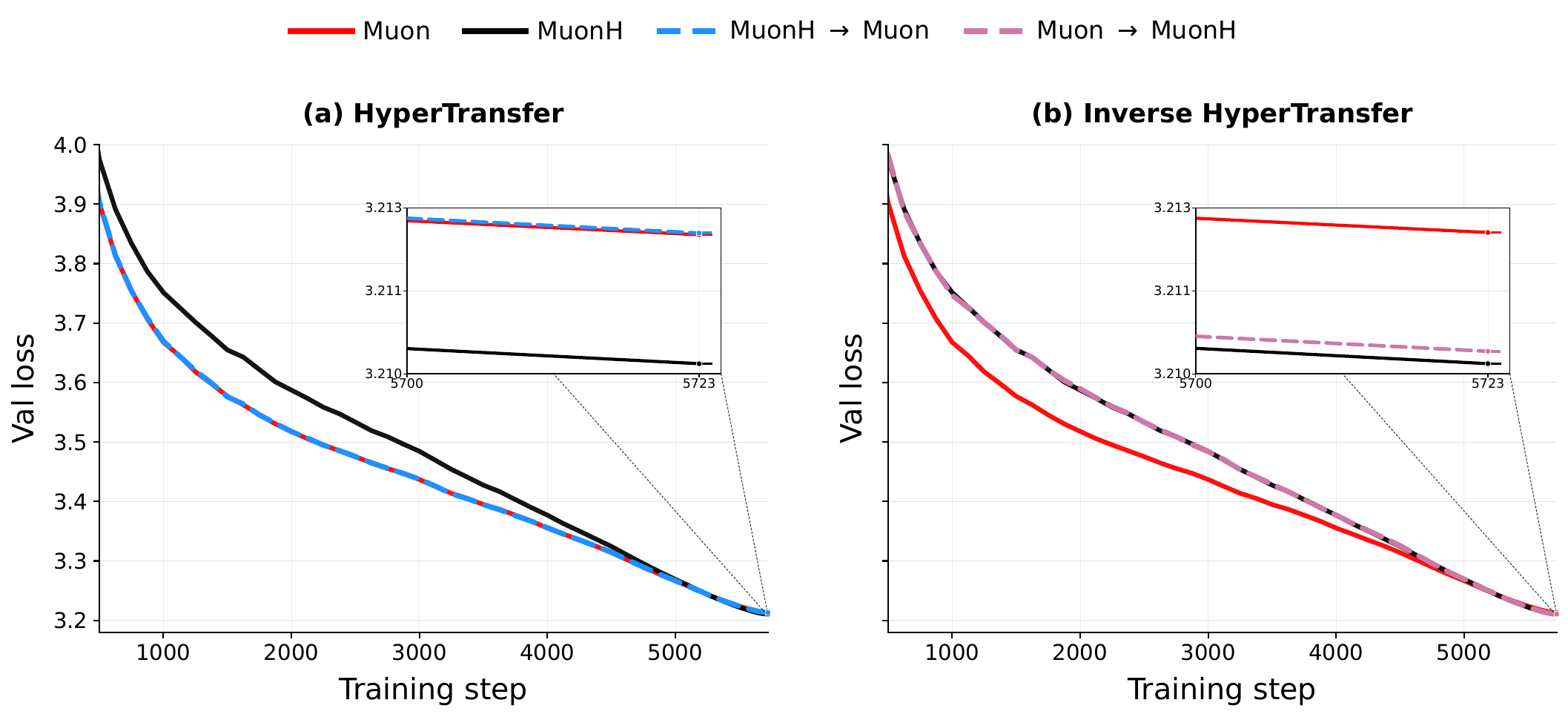}
    \caption{Bidirectional transfer between independently tuned Muon and MuonH baselines on the scale-invariant model $\mathcal{M}_1$ under a linear learning-rate schedule decaying to zero. The two baselines exhibit distinct training dynamics: Muon reduces the validation loss substantially faster during the early stage of training, whereas MuonH eventually overtakes Muon and achieves a lower final loss. In panel (a), HyperTransfer enables MuonH to reproduce the target Muon trajectory. In panel (b), Inverse HyperTransfer enables Muon to reproduce the target MuonH trajectory. }
    \label{fig1}
\end{figure}

\subsection{Empirical Validation on Strictly Scale-Invariant Networks}
\label{c1}
We empirically test the theoretical correspondence established in Section~\ref{thg} on the strictly scale-invariant network \(\mathcal{M}_1\). Using independently tuned Muon and MuonH reference runs under the same linear decay-to-zero schedule form, we evaluate HyperTransfer
and Inverse HyperTransfer in both directions. Further experimental details are provided in Section~\ref{Exp}.

Figure~\ref{fig1} reports the bidirectional transfer results. Although the independently tuned Muon and MuonH baselines exhibit distinct validation-loss trajectories, HyperTransfer enables MuonH to reproduce the Muon trajectory in panel~(a), while Inverse HyperTransfer enables Muon to reproduce the MuonH trajectory in panel~(b). In both directions, the transferred loss curves remain closely aligned with their respective reference trajectories throughout training, providing empirical support for the trajectory-matching.

For the HyperTransfer, Figure~\ref{fig:internal-alignment-main} further examines the parameter-wise dynamical alignment. The left panel evaluates the proxy-norm recursion in Equation~\eqref{up} by comparing the proxy norm \(s_t\) maintained by HyperTransfer with the parameter norm \(\lVert w_t\rVert\) measured from the target Muon run, while the right panel evaluates the effective-learning-rate transformation in Equation~\eqref{elr_aglin} by comparing the effective learning rates produced by HyperTransfer and measured from Muon. The close agreement in both panels shows that HyperTransfer successfully aligns the target norm evolution and effective learning rate dynamics. Complete parameter-wise results are provided in Figures~\ref{fig:scaleinv-norm0} and \ref{fig:scaleinv-eff-lr0}.

\begin{figure}[!t]
    \centering
    \includegraphics[width=1\linewidth]{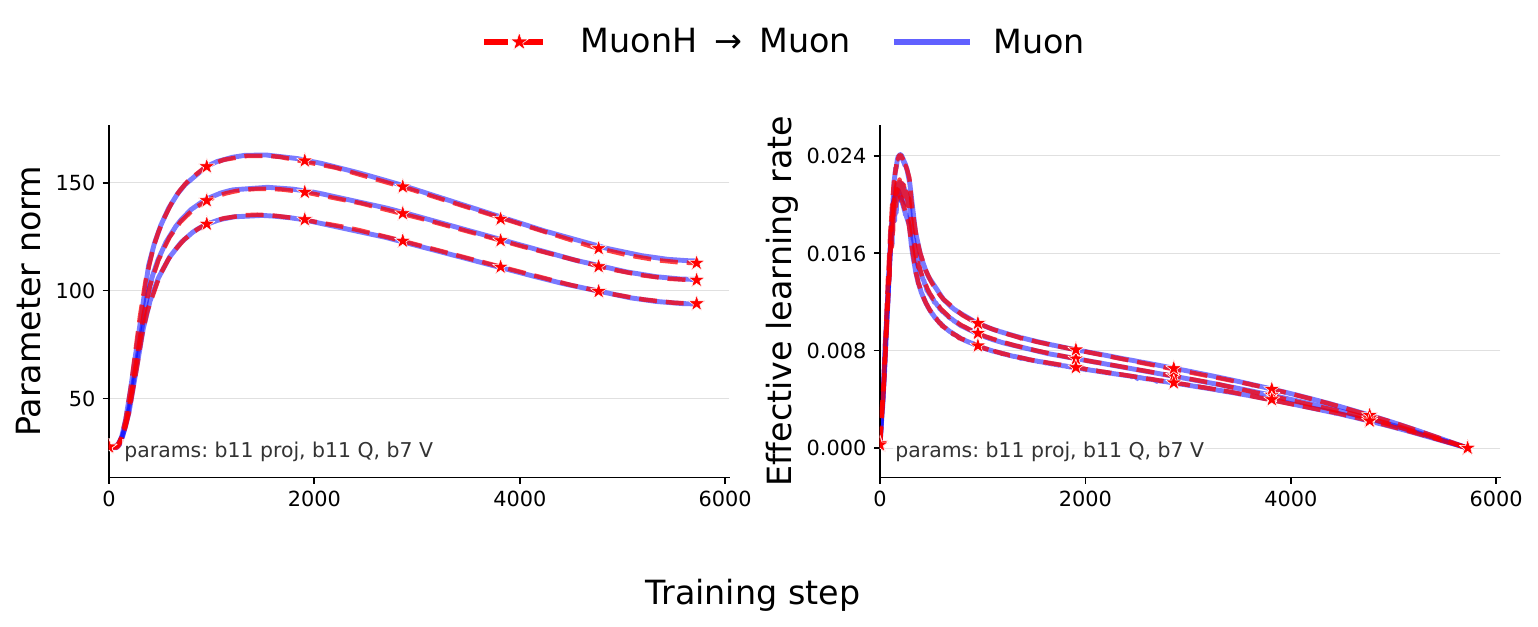}
    \caption{Parameter-wise dynamical alignment under HyperTransfer on the strictly scale-invariant network \(\mathcal{M}_1\) with a linear learning-rate schedule decaying to zero. Left: the proxy norm \(s_t\) versus the parameter norm \(\lVert w_t\rVert\) along the target Muon run. Right: the effective learning rate estimated by HyperTransfer versus that measured from Muon. In both panels, the curves closely overlap across representative hidden matrices, showing that HyperTransfer aligns the norm and effective-learning-rate
dynamics of the target optimizer.}
    \label{fig:internal-alignment-main}
\end{figure}

\section{Hypertransfer Beyond Strict Scale Invariance}
\label{sec:hypertransfer-plus}

\subsection{Hypertransfer+}

The HyperTransfer construction in Section~\ref{sec:hypertransfer} relies on strict scale invariance. In particular, the gradient at Base parameter can be obtained from the gradient at the Hyperball parameter by a rescaling. This property does not hold for the non-scale-invariant network $\mathcal{M}_2$, where $f(cw)\neq f(w)$ in general. Consequently, applying the standard HyperTransfer gradient rescaling is no longer sufficient.

One of the key differences between HyperTransfer+ and HyperTransfer lies in where the gradient is evaluated. HyperTransfer evaluates \(\nabla f(w_t^H)\) at the Hyperball parameter \(w_t^H\) and rescales it as \(\frac{R}{s_t}\nabla f(w_t^H)\), whereas HyperTransfer+ directly computes the gradient at
$
\widetilde{w}_t^H=\frac{s_t}{R} w_t^H.
$
Specifically, at iteration $t$, we form the Base-scale representative and evaluate
the gradient at this point as follows:
\begin{equation}
    \widetilde{w}_t^H
    =
    \frac{s_t}{R}w_t^H,
    \qquad
    \widetilde{g}_t^H
    =
    \nabla f(\widetilde{w}_t^H).
    \label{eq:htplus-representative}
\end{equation}

The gradient $\widetilde{g}_t^H$ is used directly in the same momentum
recursion as the Base Optimizer. The resulting update direction is
denoted by $u_t^H$. The Hyperball learning rate is then chosen to match the
effective angular update of the Base Optimizer:
\begin{equation}
    \eta_t^H
    =
    \frac{\eta_t\|u_t^H\|}
    {(1-\eta_t\lambda)s_t}.
    \label{eq:htplus-learning-rate}
\end{equation}
The Hyperball parameter is updated using the same normalized update rule as
in Section~\ref{sec:hypertransfer}. The proxy norm is updated to track the
norm of the corresponding Base Optimizer parameter. Its recursion is given by
\(
    s_{t+1}^2
    =
    (1-\eta_t\lambda)^2s_t^2
    +\eta_t^2\|u_t^H\|^2 
    -2\eta_t(1-\eta_t\lambda)
    \frac{s_t}{R}
    \langle u_t^H,w_t^H\rangle .
\)

When $\widetilde{w}_t^H$ coincides with the corresponding Base Optimizer
parameter, the two optimizers receive identical gradient inputs and therefore
produce the same momentum state and update direction. Equation~\eqref{eq:htplus-learning-rate}
then matches the effective angular update required by
Proposition~\ref{prop1}. Thus, HyperTransfer+ preserves the correspondence
between the Base parameter and the Hyperball parameter through the
proxy norm recursion.

HyperTransfer+ aligns the gradient evaluation points, optimizer states, and effective update sizes. This provides a practical extension of HyperTransfer to non-scale-invariant networks. The complete HyperTransfer+ procedure is summarized in
Algorithm~\ref{alg3}.  

The corresponding Inverse HyperTransfer+ procedure is given in Appendix~\ref{app:inverse-hypertransfer-plus}. It uses a Base Optimizer to approximate the training dynamics of Hyperball on the non-scale-invariant network $\mathcal{M}_2$, providing the inverse counterpart of HyperTransfer+.

\begin{algorithm2e}[H]
\scriptsize
\caption{HyperTransfer+}
\label{alg3}
\SetAlgoLined

{
\KwIn{
Loss function $f$; initial parameter $w_0^H$;
reference learning-rate schedule $\{\eta_t\}_{t=0}^{T-1}$;
momentum coefficient $\beta$; weight decay coefficient $\lambda$;
number of iterations $T$
}
\KwOut{Updated parameter $w_T^H$, proxy norm $s_T$}

Initialize \(w_0^H\), \(R\leftarrow\|w_0^H\|\), and \(s_0\leftarrow R\)\;

\For{$t = 0,1,\dots,T-1$}{

Evaluate the model at the proxy representative and compute gradient: \textbf{\textcolor{blue}{
\(g_t^H\leftarrow\nabla f(\frac{s_t}{R}w_t^H)\)\;}}

Set the transferred gradient: \textbf{\textcolor{blue}{
\(\widetilde g_t^H\leftarrow g_t^H\)\;}}

Update momentum: \(m_t^H\leftarrow \widetilde g_t^H\) if \(t=0\); otherwise \(m_t^H\leftarrow \beta m_{t-1}^H+(1-\beta)\widetilde g_t^H\)\;

Compute the Muon update:
\(u_t^H\leftarrow\operatorname{Newton\text{-}Schulz}(m_t^H)\)\;

Compute the transferred Hyperball angular step:
\(\eta_t^H \leftarrow \eta_t\frac{\|u_t^H\|}{(1-\eta_t\lambda)s_t}\)\;

Update parameter using Hyperball:
\(w_{t+1}^H \leftarrow
R\,(w_t^H-\eta_t^H R\,\frac{u_t^H}{\|u_t^H\|})
/\|w_t^H-\eta_t^H R\,\frac{u_t^H}{\|u_t^H\|}\|\)\;

Update proxy norm:
\(
s_{t+1}^2 \leftarrow
(1-\eta_t \lambda)^2 s_t^2
+\eta_t^2 \|u_t^H\|^2
-2\eta_t(1-\eta_t \lambda)
\frac{s_t}{R}\langle u_t^H, w_t^H \rangle
\)\;
}
\Return $w_T^H$, $s_T$\;
}
\end{algorithm2e}

\begin{figure}[!hb]
    \centering
    \includegraphics[width=.95\linewidth]{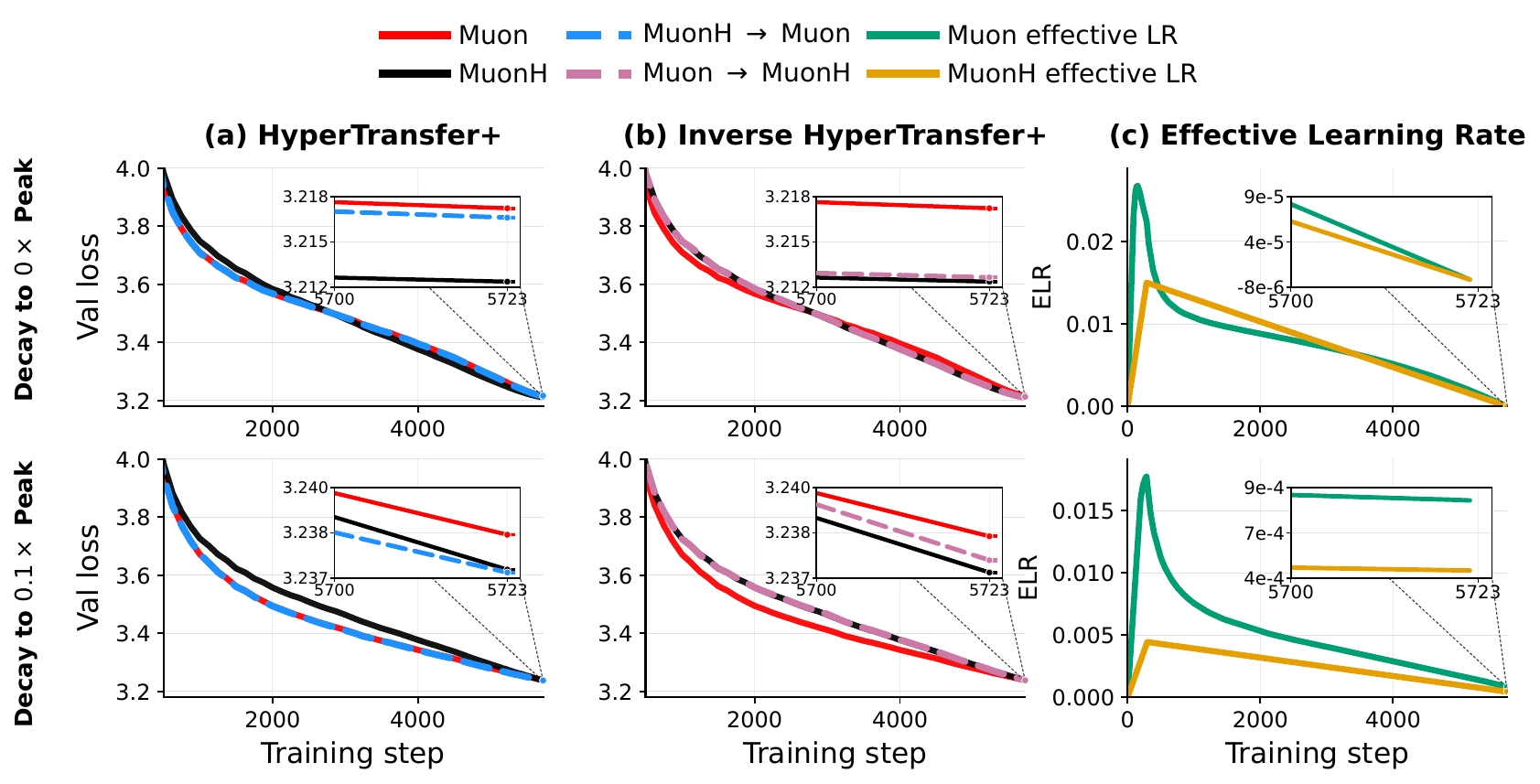}
    \caption{Bidirectional trajectory transfer on the non-scale-invariant network
\(\mathcal{M}_2\) under nominal learning-rate schedules ending at zero and at
\(0.1\times\) the peak learning rate. (a) HyperTransfer+: MuonH follows the
target Muon trajectory. (b) Inverse HyperTransfer+: Muon follows the target
MuonH trajectory. (c) Effective-learning-rate trajectories of the Muon and
MuonH baselines. The transferred trajectories closely match their respective
targets under both schedules, whereas the nominally matched schedules induce
different effective-learning-rate decay profiles.}
    \label{fig:hypertransfer-plus-grid}
\end{figure}

\subsection{Empirical Validation on Non-Scale-Invariant Networks}
\label{sec:hypertransfer-plus-experiments}

We evaluate HyperTransfer+ and Inverse HyperTransfer+ on the
non-scale-invariant network \(\mathcal{M}_2\) to determine whether they can
reproduce the training trajectories of their target optimizers without strict
scale invariance. We consider two nominal learning-rate schedules: an
end-to-end schedule that decays linearly to zero over the complete training
horizon, and a continued-pretraining schedule that retains \(0.1\times\) the
peak learning rate at the end of training. For each schedule, we independently
tune the Muon and MuonH baselines. We then apply HyperTransfer+ to make MuonH
reproduce the target Muon dynamics and apply Inverse HyperTransfer+ to make
Muon reproduce the target MuonH dynamics. Complete model configurations,
training settings, and selected hyperparameters are provided in
Appendix~\ref{Exp}.

Figure~\ref{fig:hypertransfer-plus-grid}(a) shows the HyperTransfer+ results,
whereas Figure~\ref{fig:hypertransfer-plus-grid}(b) shows the corresponding
Inverse HyperTransfer+ results. Under the decay-to-zero schedule,
HyperTransfer+ matches the Muon target at \(3.217\), with a gap below
\(5\times10^{-6}\), while Inverse HyperTransfer+ reaches \(3.213\), within
\(3\times10^{-4}\) of the MuonH target based on unrounded values. When the
schedule ends at \(0.1\times\) the peak learning rate, both baselines and both
transferred runs report a final validation loss of \(3.238\), with Inverse
HyperTransfer+ remaining within \(3.45\times10^{-4}\) of its MuonH target.
Under both schedules, the transferred trajectories closely follow their
respective targets throughout training, rather than matching only their final
validation losses.

These results demonstrate that HyperTransfer+ and Inverse HyperTransfer+
achieve accurate empirical trajectory alignment on the non-scale-invariant
network \(\mathcal{M}_2\). By correcting the gradient-evaluation points and
aligning the optimizer states and effective angular updates, the two methods
enable bidirectional transfer of optimization dynamics between a Base
Optimizer and its Hyperball counterpart even when strict scale invariance is
absent.

Figure~\ref{fig:hypertransfer-plus-grid}(c) further shows that nominal
learning-rate schedules with the same final-to-peak ratio can induce
substantially different effective-learning-rate decay profiles for Muon and
MuonH. Thus, matching nominal learning-rate schedules does not necessarily
match their induced effective-learning-rate decay. We examine this discrepancy
in the next subsection.

\begin{figure}[!ht]
    \centering
    \includegraphics[width=1\linewidth]{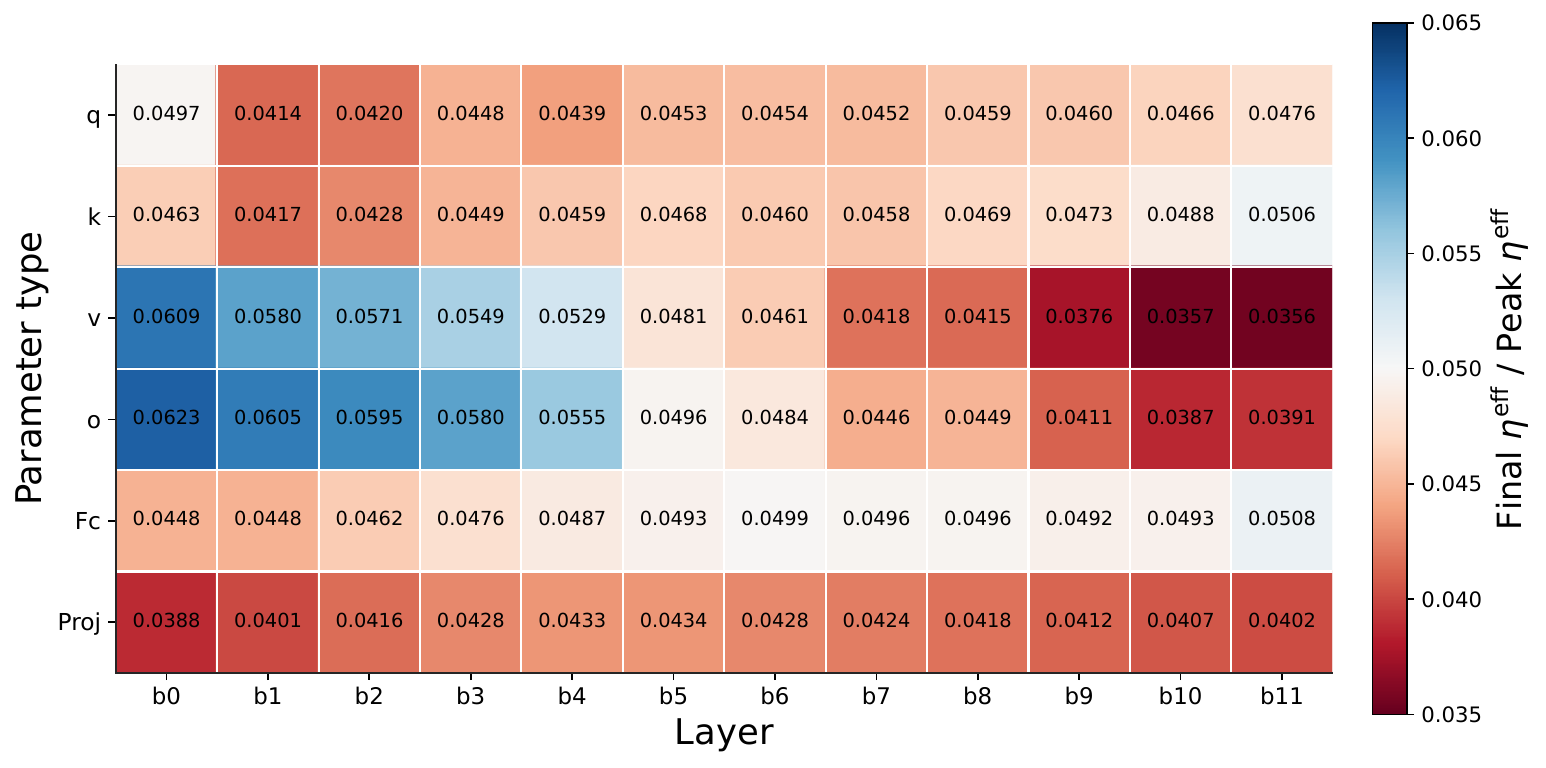}
    \caption{Parameter-wise final-to-peak effective-learning-rate ratios of
    Muon on the non-scale-invariant network $\mathcal{M}_2$ under a nominal
    final-to-peak ratio of $0.1$. Rows correspond to
    parameter types and columns correspond to Transformer blocks, with each
    cell reporting the ratio for one hidden matrix. The average across all
    hidden matrices is approximately $4.7\times10^{-2}$, showing that the
    nominal decay ratio does not equal the effective-learning-rate decay ratio
    for Muon.}
    \label{heat2}
\end{figure}

\subsection{Nominal versus Effective Learning-Rate Decay}
\label{sec_nominal_vs_elr}

Although Section~\ref{sec:hypertransfer-plus-experiments} demonstrates accurate bidirectional trajectory transfer on the non-scale-invariant network \(\mathcal{M}_2\), the rightmost column of Figure~\ref{fig:hypertransfer-plus-grid}(c) reveals an additional discrepancy: Muon and MuonH can induce substantially different effective-learning-rate decay profiles even when their nominal schedules have the same final-to-peak ratio. We therefore examine the relationship between nominal and effective learning-rate decay. 

We first set the nominal final-to-peak learning rate ratio to \(0.1\) for both optimizers and tune their peak learning rates independently. For Muon, we compute the final-to-peak ratio of the effective learning rate for each hidden weight matrix. As shown in Figure~\ref{heat2}, the average ratio across all hidden matrices is approximately \(4.7\times10^{-2}\), substantially below the nominal ratio of \(0.1\). By contrast, because MuonH's nominal learning rate directly determines its effective angular update, its effective final-to-peak ratio remains \(0.1\) by construction. Thus, matching the nominal final-to-peak ratios does not match the effective-learning-rate decay induced by the two optimizers.

Although both nominal schedules end at \(0.1\times\) their peak learning
rates, their effective final-to-peak ratios are not aligned: the
parameter-averaged ratio is approximately \(4.7\times10^{-2}\) for Muon, but
remains \(0.1\) for MuonH by construction. Matching the nominal ratio therefore
does not control effective learning-rate decay. We address this mismatch by
setting the MuonH final-to-peak ratio to \(4.7\times10^{-2}\) and independently
tuning its peak learning rate.

As shown in Figure~\ref{fig:effective-lr-matched-transfer}(a), after matching
the effective final-to-peak ratios, MuonH achieves a final validation loss of
\(3.221\), compared with \(3.238\) for the independently tuned Muon baseline.
Thus, the original ratio mismatch alone does not explain the performance gap:
the full effective-learning-rate trajectories can still differ despite matched
endpoint ratios, as shown in
Figure~\ref{fig:effective-lr-matched-transfer}(b).

We then apply Inverse HyperTransfer+ to transfer this MuonH trajectory to Muon.
The transferred Muon also reaches \(3.221\), within \(3\times10^{-5}\) of the
MuonH target, and outperforms the conventional Muon baseline. Inverse
HyperTransfer+ therefore induces an improved Muon learning-rate schedule from
the MuonH dynamics. These results show that optimizer comparisons must account
for the full effective-learning-rate trajectory, rather than only a scalar
decay ratio.

\begin{figure}[!ht]
    \centering
    \includegraphics[width=1\linewidth]{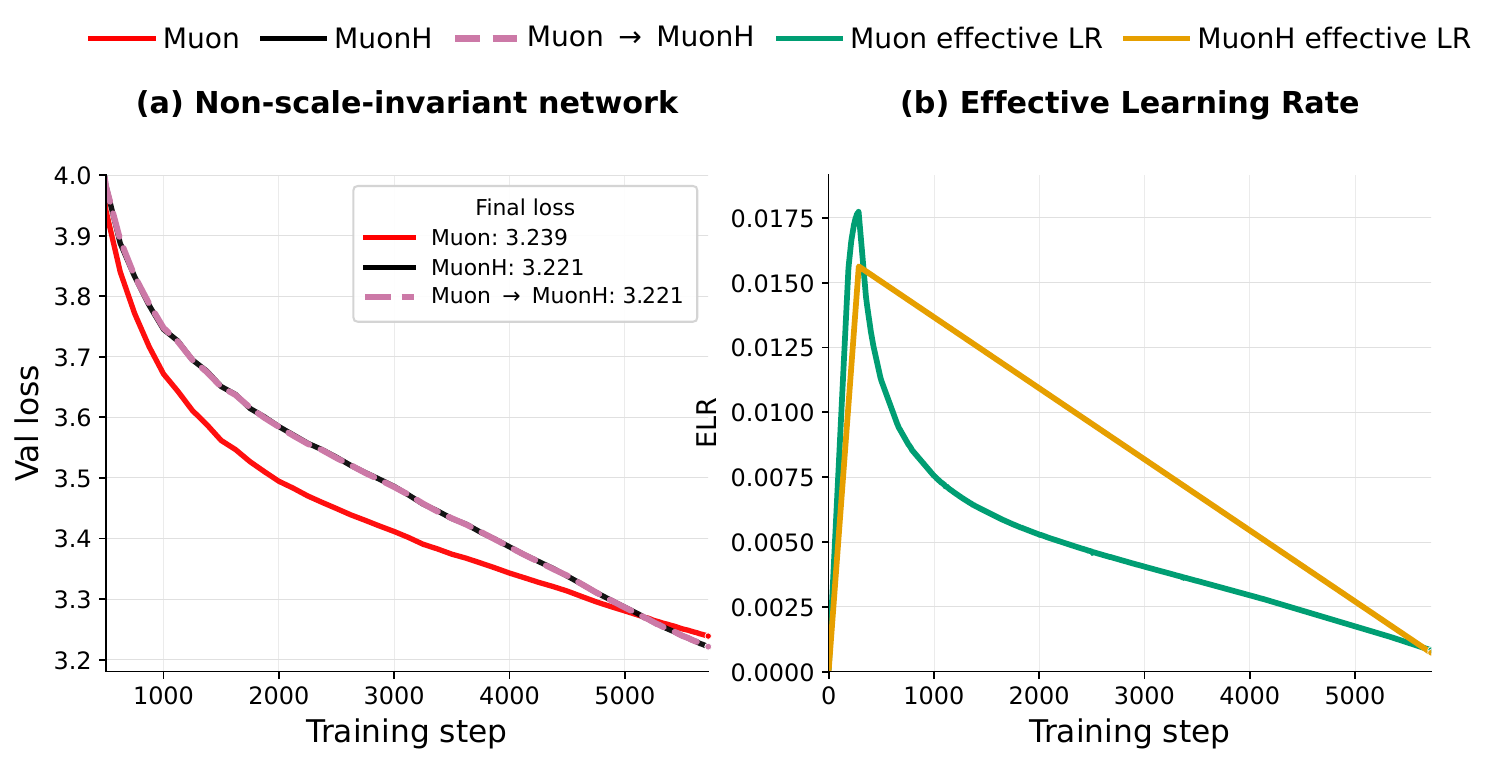}
    \caption{Transfer under matched final-to-peak effective-learning-rate
    ratios on the non-scale-invariant network $\mathcal{M}_2$. The peak
    learning rates are searched separately, while the effective final-to-peak
    ratio is matched across the two optimizers. The left panel
    compares the Muon baseline, the MuonH baseline, and the Muon trajectory
    obtained through Inverse HyperTransfer+. The right panel compares the
    corresponding Muon and MuonH effective-learning-rate schedules. The
    effective-learning-rate decay ratio of MuonH is matched to the average
    ratio measured for Muon, approximately $4.7\times10^{-2}$. The transferred
    Muon trajectory remains close to the target MuonH trajectory under
    effective-learning-rate alignment.}
    \label{fig:effective-lr-matched-transfer}
\end{figure}

\section{Conclusion}

In this work, we establish a theoretical equivalence between Base Optimizer and Hyperball through the lens of optimization dynamics. On scale-invariant network \(\mathcal{M}_1\), we show that the apparent difference between norm-constrained Hyperball and unconstrained Base Optimizer can be resolved by matching the angular update. Based on this insight, we propose HyperTransfer and its inverse mapping, enabling bidirectional transfer of optimization dynamics between Base Optimizer and Hyperball. We further extend both mappings to non-scale-invariant network \(\mathcal{M}_2\) through HyperTransfer+ and Inverse HyperTransfer+, which achieve accurate dynamical alignment beyond the strict scale-invariance assumption. Extensive experiments demonstrate that the proposed methods reliably recover target loss trajectories and reveal that Hyperball behavior is largely governed by its induced effective-learning-rate schedule and optimizer-state evolution. Our analysis also uncovers discrepancies between nominal and effective learning rates, highlighting the importance of effective-learning-rate alignment for fair comparisons between optimization methods.

\newpage

\bibliographystyle{iclr2027_conference}

\bibliography{sample}

\newpage
\clearpage
\appendix

\section{Inverse HyperTransfer and Inverse HyperTransfer+}
\subsection{Inverse HyperTransfer}
\label{app:inverse-hypertransfer}
Inverse HyperTransfer maps a prescribed Hyperball learning-rate schedule to
the nominal learning rate of a Base Optimizer. Unlike HyperTransfer, it does
not require a proxy norm because the Base Optimizer parameter norm
$\|w_t\|$ is directly available during optimization.

For the scale-invariant network $\mathcal{M}_1$, the gradient supplied to the
Base Optimizer is first rescaled to match the gradient evaluated by the
Hyperball optimizer. Specifically, using
$\nabla f(cw)=\nabla f(w)/c$, we set
$\widetilde{g}_t=(\|w_t\|/R)g_t$ and use $\widetilde{g}_t$ in the momentum
recursion. The update direction obtained from this recursion is denoted by
$u_t$.

Given a reference Hyperball learning rate $\eta_t^H$, the Base Optimizer
learning rate is obtained by solving the effective-learning-rate relation.
The resulting transformation is:
\begin{equation}
    \eta_t
    =
    \frac{\eta_t^H\|w_t\|}
    {\|u_t\|+\eta_t^H\lambda\|w_t\|}.
    \label{eq:inverse-learning-rate}
\end{equation}

The Base Optimizer then applies
$w_{t+1}=(1-\eta_t\lambda)w_t-\eta_tu_t$. Since the transferred gradient
produces $u_t=u_t^H$ and Equation~\eqref{eq:inverse-learning-rate} gives the
same angular update as the reference Hyperball optimizer, the normalized
parameters remain aligned. Therefore, on the scale-invariant network,
Inverse HyperTransfer reproduces the Hyperball loss trajectory while
allowing the parameter norm to evolve according to the Base Optimizer update.

The complete procedure is summarized in Algorithm~\ref{alg:inverse-hypertransfer}.

\begin{algorithm2e}[H]
\scriptsize
\caption{Inverse HyperTransfer}
\label{alg:inverse-hypertransfer}
\SetAlgoLined

{
\KwIn{
Loss function $f$; shared initial parameter $w_0=w_0^H$;
reference Hyperball learning-rate schedule
$\{\eta_t^H\}_{t=0}^{T-1}$;
momentum coefficient $\beta$; weight decay coefficient $\lambda$;
number of iterations $T$
}
\KwOut{Updated parameter $w_T$}

Initialize \(w_0\) and \(R\leftarrow\|w_0\|=\|w_0^H\|\)\;

\For{$t=0,1,\dots,T-1$}{
    Compute the gradient:
    $g_t\leftarrow\nabla f(w_t)$\;

    Set the transferred gradient:
    $\widetilde g_t\leftarrow \frac{\|w_t\|}{R}g_t$\;

    Update momentum: $m_t\leftarrow \widetilde g_t$ if $t=0$; otherwise $m_t\leftarrow \beta m_{t-1}+(1-\beta)\widetilde g_t$\;

    Compute the Muon update:
    $u_t\leftarrow\operatorname{Newton\text{-}Schulz}(m_t)$\;

    Transfer the Hyperball learning rate to the Base Optimizer:
    $\eta_t\leftarrow \eta_t^H\frac{\|w_t\|}{\|u_t\|+\eta_t^H\lambda\|w_t\|}$\;

    Update the parameter:
    $w_{t+1}\leftarrow
    (1-\eta_t\lambda)w_t-\eta_tu_t$\;
}

\Return $w_T$\;
}
\end{algorithm2e}

\subsection{Inverse HyperTransfer+}
\label{app:inverse-hypertransfer-plus}

Inverse HyperTransfer+ extends the inverse construction to the
non-scale-invariant network $\mathcal{M}_2$. Because the gradient at a
rescaled parameter cannot be recovered by a scalar rescaling, the Base
Optimizer evaluates the model directly at the Hyperball-scale
representative. At iteration $t$, this representative and the gradient used
by the Base Optimizer are:
\begin{equation}
    \widetilde{w}_t
    =
    \frac{R}{\|w_t\|}w_t,
    \qquad
    g_t
    =
    \nabla f(\widetilde{w}_t).
    \label{eq:inverse-plus-gradient}
\end{equation}

The gradient $g_t$ is used directly in the Base Optimizer momentum recursion,
without an additional scale-based rescaling. The update direction is denoted
by $u_t$. Given the prescribed Hyperball learning rate $\eta_t^H$, the Base
Optimizer learning rate is still obtained from the inverse transformation:
\begin{equation}
    \eta_t
    =
    \frac{\eta_t^H\|w_t\|}
    {\|u_t\|+\eta_t^H\lambda\|w_t\|}.
    \label{eq:inverse-plus-learning-rate}
\end{equation}

The Base Optimizer is then updated according to
$w_{t+1}=(1-\eta_t\lambda)w_t-\eta_tu_t$. By evaluating the gradient at
$\widetilde{w}_t=(R/\|w_t\|)w_t$, Inverse HyperTransfer+ aligns the gradient
input and update direction with those of the reference Hyperball optimizer.
The parameter norm itself remains unconstrained and evolves naturally under
the Base Optimizer update.

Since $\mathcal{M}_2$ is not scale-invariant, the construction does not imply
exact equality of the loss values at $w_t$ and $w_t^H$. Instead,
Inverse HyperTransfer+ aligns the designated gradient evaluation points and
the effective angular updates, providing the inverse counterpart of
HyperTransfer+ for non-scale-invariant networks.

The complete procedure is summarized in
Algorithm~\ref{alg4}.

\begin{algorithm2e}[H]
\scriptsize
\caption{Inverse HyperTransfer+}
\label{alg4}
\SetAlgoLined

{
\KwIn{
Loss function \(f\); shared initialization \(w_0=w_0^H\);
Hyperball learning-rate schedule \(\{\eta_t^H\}_{t=0}^{T-1}\);
momentum coefficient \(\beta\); weight-decay coefficient \(\lambda\);
number of iterations \(T\)
}
\KwOut{Updated parameter \(w_T\)}

Initialize \(w_0\) and \(R\leftarrow\|w_0\|=\|w_0^H\|\)\;

\For{$t=0,1,\dots,T-1$}{

Evaluate the model at the Hyperball-scale representative and compute: 
\textbf{\textcolor{blue}{\(
g_t\leftarrow
\nabla f\!\left(\frac{R}{\|w_t\|}w_t\right)
\)\;}}

Set the transferred gradient:
\textbf{\textcolor{blue}{\(
\widetilde g_t\leftarrow g_t
\)\;}}

Update momentum:
\(
m_t\leftarrow\widetilde g_t
\)
if \(t=0\); otherwise
\(
m_t\leftarrow
\beta m_{t-1}+(1-\beta)\widetilde g_t
\)\;

Compute the Muon update:
\(
u_t\leftarrow
\operatorname{Newton\text{-}Schulz}(m_t)
\)\;

Transfer the Hyperball learning rate to the Base Optimizer:
\(
\eta_t\leftarrow
\eta_t^H
\frac{\|w_t\|}
{\|u_t\|+\eta_t^H\lambda\|w_t\|}
\)\;

Update the parameter:
\(
w_{t+1}\leftarrow
(1-\eta_t\lambda)w_t-\eta_tu_t
\)\;
}

\Return \(w_T\)\;
}
\end{algorithm2e}

\section{Proof of Proposition}

\subsection{Proof of Proposition~\ref{prop1}}
\begin{proof}
Let $R=\|w_t^H\|$ denote the fixed Hyperball radius. We prove the result by
induction. At $t=0$, the
initialization gives $w_0=w_0^H$, so the normalized parameters are aligned.

Assume that $w_t=(\|w_t\|/R)w_t^H$ at iteration $t$. Let
$\bar w_{t+1}^H$ denote the Hyperball iterate before radial normalization.
The angular-update condition gives:
\begin{equation}
    \begin{aligned}
    \bar{w}_{t+1}^H & = w_t^H-\eta_t^H R\frac{u_t^H}{\|u_t^H\|} \\
    & =\frac{R}{\|w_t\|} \left(w_t-\frac{\eta_t}{1-\eta_t\lambda}u_t\right) \\
    &=\frac{R}{(1-\eta_t\lambda)\|w_t\|}w_{t+1},
    \end{aligned}
\end{equation}
where the second equality uses
$\eta_t^{H}=\eta_t\|u_t\|/((1-\eta_t\lambda)\|w_t\|)$ and the last equality
uses the Base Optimizer update
$w_{t+1}=(1-\eta_t\lambda)w_t-\eta_tu_t$. Since
$1-\eta_t\lambda>0$, the two pre-normalization
updates lie on the same positive ray. Radial normalization changes only the
norm, so $w_{t+1}/\|w_{t+1}\|=w_{t+1}^H/\|w_{t+1}^H\|$. Scale invariance
then gives $f(w_{t+1})=f(w_{t+1}^H)$, completing the induction.
\end{proof}

\subsection{Proof of Theorem~\ref{thm:hypertransfer}}
\begin{proof}
Let $R=\|w_0^H\|$ and initialize $s_0=R$. We prove the result by
induction. At $t=0$, $w_0=w_0^H=(s_0/R)w_0^H$, so the required relation
holds.

Assume that at iteration $t$, $w_t=(s_t/R)w_t^H$ and
$s_t=\|w_t\|$. By Lemma~\ref{lemma_gradient}, the transferred Hyperball
gradient satisfies
$\widetilde g_t^H=(R/s_t)g_t^H=g_t$. Thus the two momentum recursions receive
the same gradient sequence and produce the same update direction,
$u_t^H=u_t$.

The learning-rate rule in Equation~\eqref{elr_aglin} then satisfies the
angular-update condition in Proposition~\ref{prop1}, because:
\begin{equation}
    \eta_t^H\frac{u_t^H}{\|u_t^H\|}
    =
    \frac{\eta_tu_t}{(1-\eta_t\lambda)s_t}
    =
    \eta_t^{\mathrm{eff}}\frac{u_t}{\|u_t\|}.
\end{equation}
Proposition~\ref{prop1} therefore gives
$w_{t+1}/\|w_{t+1}\|=w_{t+1}^H/R$. It remains to show that the proxy norm
has the correct value at the next iteration. The recursion in the theorem
gives:
\begin{equation}
    \begin{aligned}
    s_{t+1}^2
    &=(1-\eta_t\lambda)^2s_t^2
      +\eta_t^2\|u_t^H\|^2
      -2\eta_t(1-\eta_t\lambda)
       \frac{s_t}{R}\langle u_t^H,w_t^H\rangle \\
    &=(1-\eta_t\lambda)^2\|w_t\|^2
      +\eta_t^2\|u_t\|^2
      -2\eta_t(1-\eta_t\lambda)\langle u_t,w_t\rangle \\
    &=\|(1-\eta_t\lambda)w_t-\eta_tu_t\|^2 \\
    &=\|w_{t+1}\|^2.
    \end{aligned}
\end{equation}
Since both quantities are nonnegative, $s_{t+1}=\|w_{t+1}\|$. Hence
$w_{t+1}=(s_{t+1}/R)w_{t+1}^H$, and the induction is complete.
\end{proof}

\subsection{Proof of Lemma~\ref{lemma_gradient}}
\begin{proof}
Since $f(cw)=f(w)$ for every $c>0$, differentiating both sides with respect
to $w$ gives:
\begin{equation}
    c\nabla f(cw)=\nabla f(w).
\end{equation}
Therefore, $\nabla f(cw)=\nabla f(w)/c$.
\end{proof}

\subsection{Proof of HyperTransfer+}
\begin{proposition}
\label{prop:hypertransfer-plus}
Assume that $w_0=w_0^H$, $R=\|w_0^H\|$, and $s_0=R$. Suppose that the
Base Optimizer and Hyperball use identical optimizer-state initializations,
the same optimizer-state recursion, and the same sequence of minibatches.
If HyperTransfer+ evaluates the gradient at the Base-scale representative
$(s_t/R)w_t^H$, uses the learning-rate rule
$\eta_t^H=\eta_t\|u_t^H\|/((1-\eta_t\lambda)s_t)$, and updates $s_t$ by
Equation \eqref{up}, then
$s_t=\|w_t\|$ and $w_t=(s_t/R)w_t^H$ for every iteration. Consequently,
the loss evaluated at the Base-scale representative satisfies
$f(w_t)=f(\frac{s_t}{R}w_t^H)$ for every iteration, including for a
non-scale-invariant $f$.
\end{proposition}
\begin{proof}
Let $R=\|w_0^H\|$ and initialize $s_0=R$. HyperTransfer+ evaluates the
gradient at the Base-scale representative $(s_t/R)w_t^H$. At $t=0$, this
representative equals $w_0=w_0^H$.

Assume that $(s_t/R)w_t^H=w_t$ and $s_t=\|w_t\|$. The two optimizers then
receive the same gradient input and produce the same update direction,
$u_t^H=u_t$. The HyperTransfer+ learning-rate rule gives:
\begin{equation}
    \eta_t^H\frac{u_t^H}{\|u_t^H\|}
    =
    \frac{\eta_tu_t}{(1-\eta_t\lambda)s_t}.
\end{equation}
Substitution into the Hyperball update shows that its pre-normalization
iterate is a positive scalar multiple of
$w_{t+1}=(1-\eta_t\lambda)w_t-\eta_tu_t$. Radial normalization therefore
places $w_{t+1}^H$ on the same ray as $w_{t+1}$. The proxy norm recursion is
the squared norm of the corresponding Base update, so
$s_{t+1}=\|w_{t+1}\|$ and
$(s_{t+1}/R)w_{t+1}^H=w_{t+1}$. This closes the induction.

For the non-scale-invariant network, this proof aligns the designated
gradient evaluation points and optimizer updates; it does not imply exact
equality of $f(w_t)$ and $f(w_t^H)$.
\end{proof}

\subsection{Proof for Adaptive Optimizers}
\begin{proposition}
\label{prop:adaptive-hypertransfer}
Assume that $f$ is scale-invariant, the Base Optimizer and Hyperball are
identically initialized, and they use the same minibatches, adaptive-state
recursions, and hyperparameters. If the transferred gradient is
$\widetilde g_t^H=(R/s_t)g_t^H$ and the Hyperball learning rate is chosen as
$\eta_t^H=\eta_t\|u_t^H\|/((1-\eta_t\lambda)s_t)$, then the adaptive states
and update directions of the two optimizers coincide. Consequently,
$w_t/\|w_t\|=w_t^H/\|w_t^H\|$ and
$f(w_t)=f(w_t^H)$ for every iteration, provided that $s_t$ is updated by
the proxy norm recursion corresponding to the Base Optimizer update.
\end{proposition}
\begin{proof}
The same induction extends to adaptive optimizers such as AdamW. Let
$m_t$ and $v_t$ denote the first- and second-moment states, respectively.
Assume that the two optimizers are identically initialized and that
$w_t=(s_t/R)w_t^H$. By Lemma~\ref{lemma_gradient}, the transferred
Hyperball gradient satisfies $\widetilde g_t^H=(R/s_t)g_t^H=g_t$.

The first- and second-moment recursions therefore receive identical gradient
inputs. For AdamW, the corresponding equalities are:
\begin{equation}
    \begin{aligned}
    m_t^H
    &=\beta_1m_{t-1}^H+(1-\beta_1)\widetilde g_t^H
      =\beta_1m_{t-1}+(1-\beta_1)g_t
      =m_t, \\
    v_t^H
    &=\beta_2v_{t-1}^H+(1-\beta_2)\widetilde g_t^H\odot\widetilde g_t^H
      =\beta_2v_{t-1}+(1-\beta_2)g_t\odot g_t
      =v_t.
    \end{aligned}
\end{equation}
Consequently, the adaptive update directions are identical. For example,
including the usual bias-correction factors, AdamW gives:
\begin{equation}
    u_t^H
    =
    \frac{\widehat m_t^H}{\sqrt{\widehat v_t^H}+\epsilon}
    =
    \frac{\widehat m_t}{\sqrt{\widehat v_t}+\epsilon}
    =u_t.
\end{equation}
The Hyperball learning-rate rule consequently satisfies:
\begin{equation}
    \eta_t^H\frac{u_t^H}{\|u_t^H\|}
    =
    \frac{\eta_tu_t}{(1-\eta_t\lambda)s_t}
    =
    \eta_t^{\mathrm{eff}}\frac{u_t}{\|u_t\|}.
\end{equation}
Proposition~\ref{prop1} therefore preserves the alignment of the normalized
parameters. The proxy norm recursion gives
$s_{t+1}=\|w_{t+1}\|$, so the same argument applies at the next iteration.
Thus, HyperTransfer also applies to adaptive optimizers, provided that the
two optimizers use the same adaptive-state recursion, initialization, and
gradient inputs.
\end{proof}

\section{Experimental Setting}
\label{Exp}

\subsection{Common Training Setup}

All reported experiments use the same GPT-2 language-model training setup unless otherwise stated. The model has 124M parameters, vocabulary size 50304, 12 transformer blocks, model dimension 768, attention head dimension 128, and sequence length 1024. The training data are sampled from FineWeb  binary shards. Each run is trained for 5723 optimizer updates with global batch size
\(
8 \times 64 \times 1024 = 524288
\)
tokens per update, corresponding to 3B total training tokens. All runs use random seed 1234.

We use a linear learning-rate schedule with warmup fraction 0.05, giving 286 warmup steps. After warmup, the learning rate decays linearly to a final learning-rate ratio. The main final learning-rate ratios are \(0\) and \(0.1\). Validation is performed every 125 steps before the final 10\% of training and every 25 steps near the end of training.

Hidden two-dimensional weight matrices inside transformer blocks are optimized by Muon or MuonH. All remaining parameters are optimized by AdamW. AdamW uses betas \((0.9,0.95)\), \(\epsilon=10^{-10}\), weight decay 0.1. Muon uses momentum 0.95 and weight decay 0.1. Hidden weight initialization uses truncated normal initialization with standard deviation \(1/\sqrt{d_{\mathrm{in}}}\), embedding and language-model-head parameters use truncated normal initialization with standard deviation 0.006, one-dimensional parameters are initialized to zero, so RMSNorm effective gains start from one. 

In the transfer experiments, we use HyperTransfer and Inverse HyperTransfer on \(\mathcal{M}_1\), and HyperTransfer+ and Inverse HyperTransfer+ on \(\mathcal{M}_2\). HyperTransfer and HyperTransfer+ make MuonH follow the target Muon dynamics, whereas Inverse HyperTransfer and Inverse HyperTransfer+ make Muon follow the target MuonH dynamics. For each transfer run, the transferred optimizer and the target or reference optimizer start from the same initialization and use the same data order. The AdamW learning-rate schedule for all remaining parameters is also set to that of the target run, using the corresponding target AdamW peak learning rate with the same warmup and decay rule. 

\subsection{Network Variants}

We use two network variants. The first is the strictly scale-invariant network. It applies RMSNorm after each selected linear transformation and before any subsequent linear transformation or nonlinear activation, making the corresponding two-dimensional weight parameters invariant to positive rescaling. Specifically, RMSNorm is applied after the Q, K, V, attention output projection, MLP \texttt{fc}, and MLP \texttt{proj} linear layers. We use this network to test the exact theory underlying standard HyperTransfer and Inverse HyperTransfer.

The second is the non-scale-invariant network. It retains only the Q/K normalization, matching the original NanoGPT configuration. For this network, the main transfer experiments use HyperTransfer+ and Inverse HyperTransfer+, because standard HyperTransfer no longer satisfies the exact gradient-scaling identity required by the strict theory.

\begin{figure}[htbp]
    \centering
    \includegraphics[width=1\linewidth]{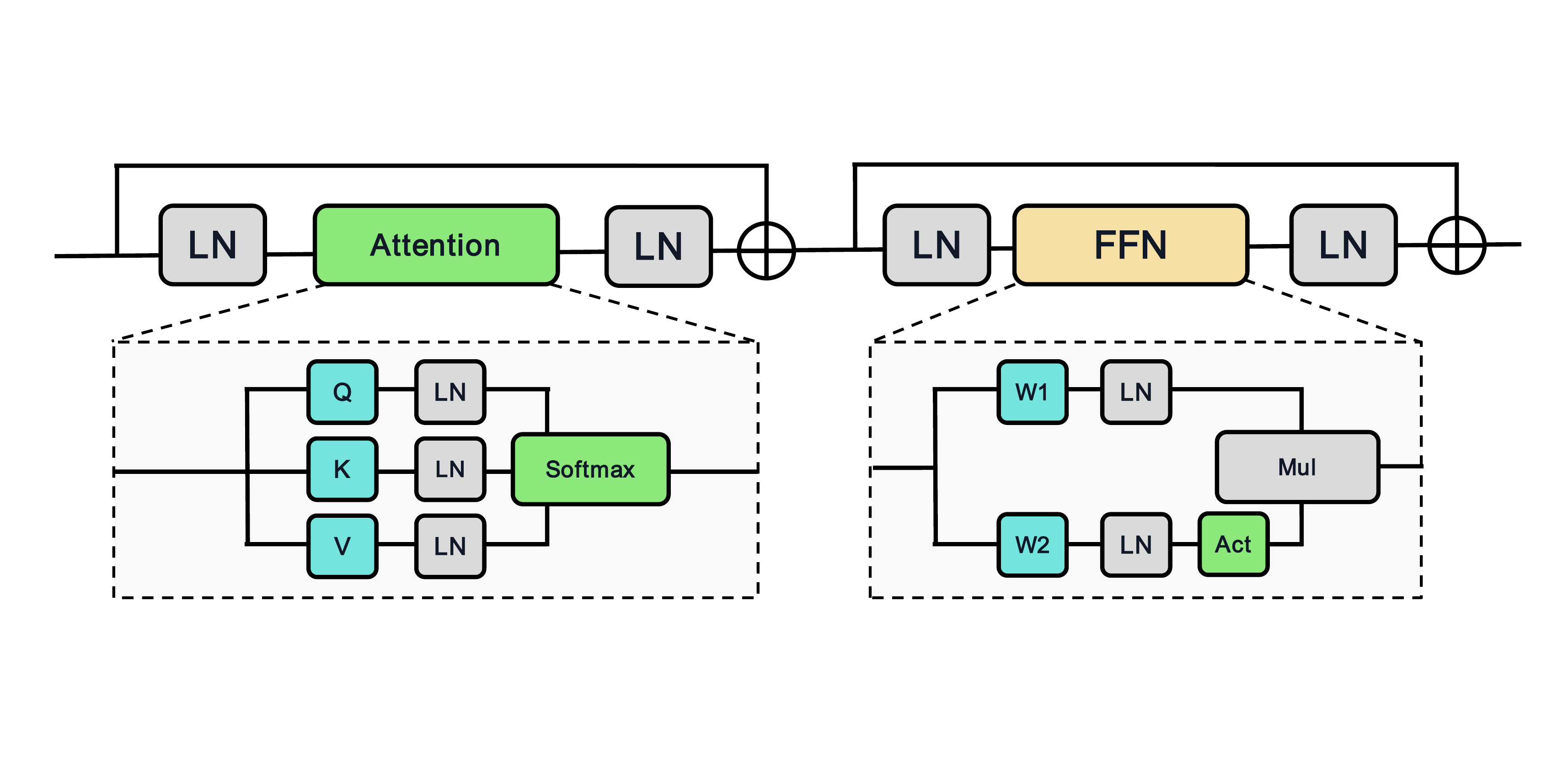}
    \caption{Architecture of the strictly scale-invariant variant. RMSNorm (denoted by LN) is applied before each attention and FFN sublayer and after the selected linear transformations within the attention and FFN modules, ensuring invariance of the corresponding weight matrices to positive rescaling. The non-scale-invariant variant follows the original NanoGPT configuration and retains only the block-level pre-norm and Q/K normalization, removing the  post-linear normalization.}
    \label{arch}
\end{figure}
\subsection{Experiment 1: Validation on Scale-Invariant Network}

{We use the scale-invariant network described above and evaluate a linear learning-rate schedule that decays to zero. We first independently tune the Muon and MuonH baselines. We then use HyperTransfer to make MuonH reproduce the Muon baseline dynamics, and Inverse HyperTransfer to make Muon reproduce the MuonH baseline dynamics. Thus, the experiment tests whether MuonH can reproduce the dynamics of Muon and whether Muon can reproduce the dynamics of MuonH.

Under this schedule, the selected Muon baseline uses peak learning rates \(\eta_{\mathrm{Muon}}=0.01\) and \(\eta_{\mathrm{AdamW}}=0.01\), and achieves a final validation loss of \(3.213\). The selected MuonH baseline uses \(\eta_{\mathrm{MuonH}}=0.015\) and \(\eta_{\mathrm{AdamW}}=0.01\), and achieves a final validation loss of \(3.210\). HyperTransfer, in which MuonH reproduces the Muon dynamics, achieves a final validation loss of \(3.213\). Inverse HyperTransfer, in which Muon reproduces the MuonH dynamics, achieves a final validation loss of \(3.210\). Using unrounded values, the final validation-loss gaps relative to the corresponding targets are on the order of \(10^{-4}\) or smaller.

For each transfer run, we record the proxy norm \(s_t\), the norm of the corresponding Base Optimizer parameter \(\|w_t\|\), the transferred effective learning rate, and the effective learning rate measured from the target Muon run for each hidden matrix. These diagnostics are shown in Figure~\ref{fig:internal-alignment-main} and in the parameter-wise plots in Figures~\ref{fig:scaleinv-norm0}--\ref{fig:scaleinv-raw-lr0}. 
}

\subsection{Experiment 2: Validation on Non-Scale-Invariant Network}

{
This experiment evaluates the non-scale-invariant network described above. In this network, only the Q/K hidden matrices retain the strict scale-invariant structure, while the V output, attention projection, MLP \texttt{fc}, and MLP \texttt{proj} matrices are scale-sensitive. Therefore, the main transfer runs use HyperTransfer+ and Inverse HyperTransfer+, rather than the standard HyperTransfer rules. 

For the schedule decaying to zero, the selected Muon baseline uses \(\eta_{\mathrm{Muon}}=0.015\) and \(\eta_{\mathrm{AdamW}}=0.015\), reaching final validation loss \(3.217\). The selected MuonH baseline uses \(\eta_{\mathrm{MuonH}}=0.015\) and \(\eta_{\mathrm{AdamW}}=0.01\), reaching final validation loss \(3.212\). HyperTransfer+ in the \(\mathrm{MuonH}\rightarrow\mathrm{Muon}\) direction reaches final loss \(3.217\), matching the Muon target within \(5\times10^{-6}\). Inverse HyperTransfer+ in the \(\mathrm{Muon}\rightarrow\mathrm{MuonH}\) direction reaches final loss \(3.213\), within \(3\times10^{-4}\) of the MuonH target.

For the schedule decaying to \(0.1\times\) the peak learning rate, the selected Muon baseline uses \(\eta_{\mathrm{Muon}}=0.0044\) and \(\eta_{\mathrm{AdamW}}=0.0044\), reaching final validation loss \(3.238\). The selected MuonH baseline uses \(\eta_{\mathrm{MuonH}}=0.0125\) and \(\eta_{\mathrm{AdamW}}=0.0064\), reaching final validation loss \(3.238\). HyperTransfer+ reaches final loss \(3.238\). Inverse HyperTransfer+ reaches final loss \(3.238\), within \(3.45\times10^{-4}\) of the MuonH target.

The first and second rows of Figure~\ref{fig:hypertransfer-plus-grid} report these two settings. The parameter-wise diagnostics for the schedule decaying to zero are shown in Figures~\ref{fig:data6-norm0-style}--\ref{fig:data6-Lr0-style}. The diagnostics for the \(0.1\times\) schedule are shown in Figures~\ref{fig:data8-norm0-style}--\ref{fig:data8-Lr0-style}. These diagnostics use the same recorded quantities as in the scale-invariant experiment, but with the HyperTransfer+ gradient-evaluation rule.
}

\subsection{Experiment 3: Effective Learning Rate Alignment}

{
This experiment studies a learning-rate schedule that decays to \(0.1\times\) the peak learning rate does not necessarily produce a final-to-peak effective-learning-rate ratio of \(0.1\). We use the Muon baselines under this schedule and compute, for each hidden weight matrix, the ratio between the final effective learning rate and the peak effective learning rate. The effective learning rate is measured as
\(
    \eta^{\mathrm{eff}}_t=\eta_t\frac{\|u_t\|}{(1-\eta_t\lambda)\|w_t\|}.
\)
Although the nominal schedule ends at \(0.1\times\) the peak learning rate, the measured effective learning-rate ratios are much smaller. Averaged over the hidden matrices, the final-to-peak ratio is \(4.7\times10^{-2}\) on the non-scale-invariant network. The corresponding parameter-wise heatmaps are shown in Figures~\ref{heat2}.

For the non-scale-invariant network, we set the MuonH's schedule to decay to \(4.7\times10^{-2}\) of the peak learning rate. The best run uses \(\eta_{\mathrm{MuonH}}=0.015\) and \(\eta_{\mathrm{AdamW}}=0.008\).

Finally, we transfer these effective-learning-rate-matched MuonH runs back to Muon. Inverse HyperTransfer+ reaches final loss \(3.221\), within \(3\times10^{-5}\) of the matched MuonH reference. These results, shown in Figure~\ref{fig:effective-lr-matched-transfer}, indicate that matching the nominal learning-rate schedule is not sufficient for a fair comparison: the induced effective learning-rate schedule may be a central part of the optimizer dynamics \citep{yang2026nonlinearitylearningratescaling}.
}

\clearpage
\section{Parameter-wise Results}
\subsection{Parameter-wise Results for Scale-Invariant Networks}
Figures~\ref{fig:scaleinv-norm0}--\ref{fig:scaleinv-raw-lr0} show the proxy norms, effective learning rates and raw learning-rate schedules for the scale-invariant network under a schedule decaying to zero.
\begin{figure}[H]
    \centering
    \includegraphics[width=.9\linewidth]{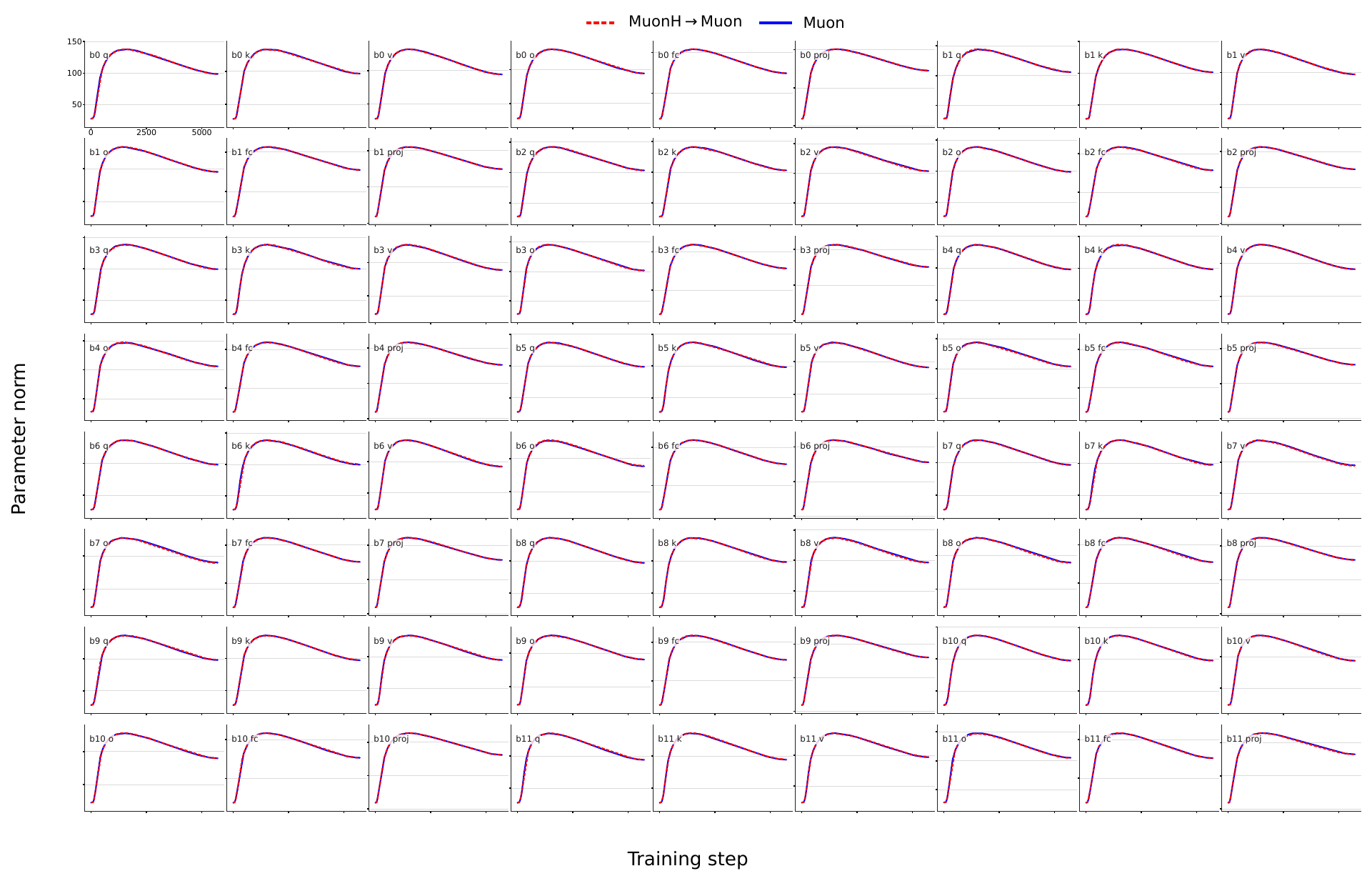}
    \caption{Parameter-wise proxy norm tracking under HyperTransfer on the scale-invariant network with a learning-rate schedule decaying to zero. Each panel corresponds to one hidden matrix. The red dashed curve shows the proxy norm \(s_t\) maintained by the \(\mathrm{MuonH}\rightarrow\mathrm{Muon}\) transfer run, while the blue solid curve shows the actual parameter norm \(\lVert w_t\rVert\) measured along the target Muon trajectory. Their close agreement demonstrates that \(s_t\) accurately recovers the parameter-norm evolution of the target optimizer.}
    \label{fig:scaleinv-norm0}
\end{figure}
\begin{figure}
    \centering
    \includegraphics[width=1\linewidth]{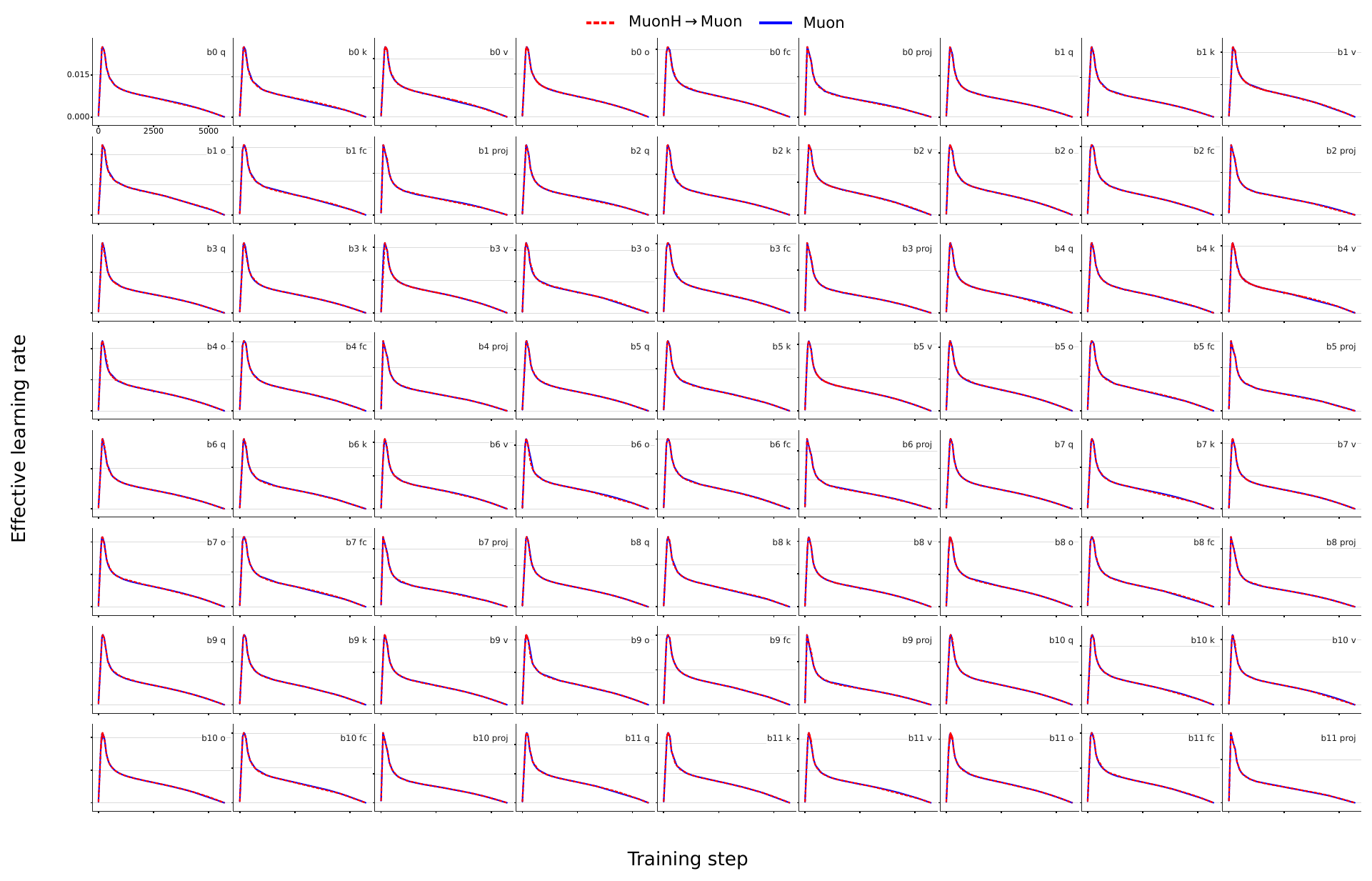}
    \caption{Parameter-wise effective-learning-rate tracking for the scale-invariant network under a schedule decaying to zero. Each small panel corresponds to one hidden matrix. The red dashed curves show the HyperTransfer estimates, while the blue curves show the corresponding Muon references.}
    \label{fig:scaleinv-eff-lr0}
\end{figure}
\begin{figure}
    \centering
    \includegraphics[width=1\linewidth]{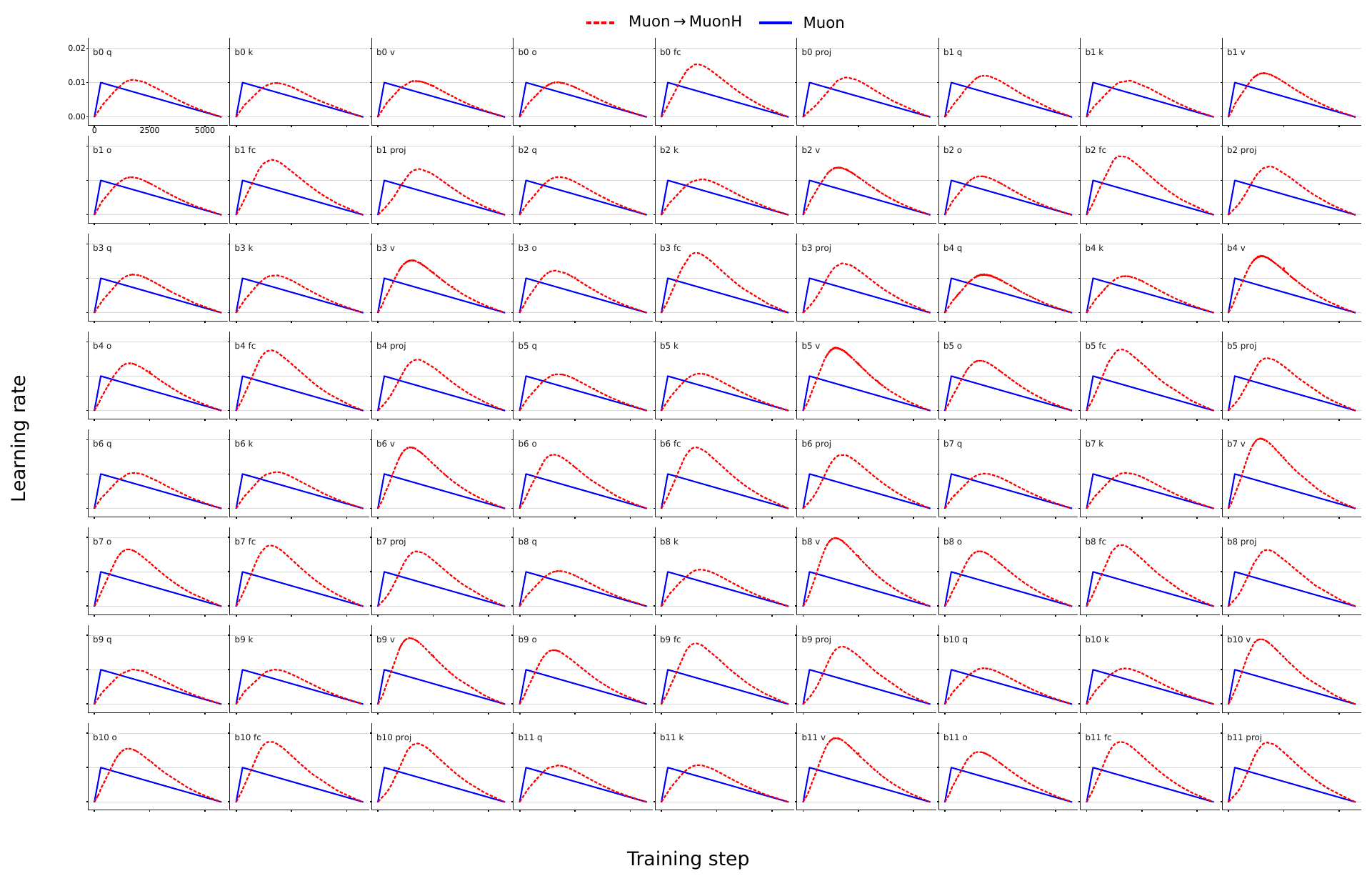}
    \caption{Parameter-wise nominal learning-rate schedules of Muon induced by Inverse HyperTransfer on the scale-invariant network, using a reference Hyperball schedule that decays to zero. Each panel corresponds to one hidden matrix. The red dashed curve shows the nominal learning rate \(\eta_t\) assigned to Muon, while the blue solid curve shows the nominal learning-rate schedule of the independently tuned Muon reference. }
    \label{fig:scaleinv-raw-lr0}
\end{figure}
\clearpage

\subsection{Parameter-wise Results on Non-Scale-Invariant Networks}
\subsubsection{Results with the learning rate decayed to zero}
Figures~\ref{fig:data6-norm0-style}--\ref{fig:data6-Lr0-style} present the
corresponding parameter-wise results for the non-scale-invariant network
$\mathcal{M}_2$ under a learning-rate schedule decaying to zero.
\begin{figure}[H]
    \centering
    \includegraphics[width=1\linewidth]{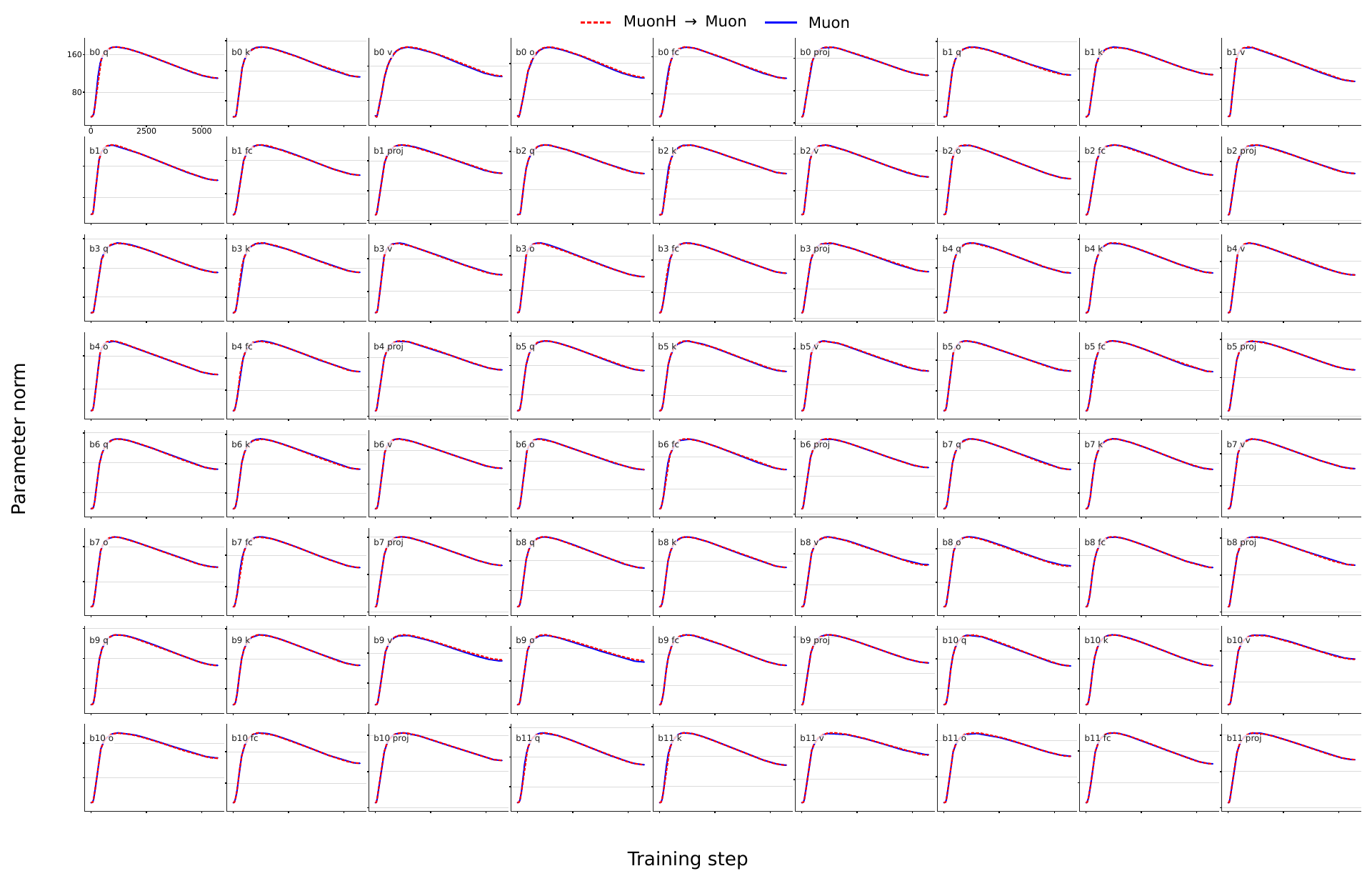}
    \caption{Parameter-wise proxy norm tracking under HyperTransfer+ on the non-scale-invariant network with a learning-rate schedule decaying to zero. Each panel corresponds to one hidden matrix. The red dashed curve shows the proxy norm \(s_t\) maintained by the \(\mathrm{MuonH}\rightarrow\mathrm{Muon}\) transfer run, while the blue solid curve shows the actual parameter norm \(\lVert w_t\rVert\) measured along the target Muon trajectory. Their close agreement demonstrates that \(s_t\) accurately recovers the parameter-norm evolution of the target optimizer.}
    \label{fig:data6-norm0-style}
\end{figure}
\begin{figure}[p]
    \centering
    \includegraphics[width=1\linewidth]{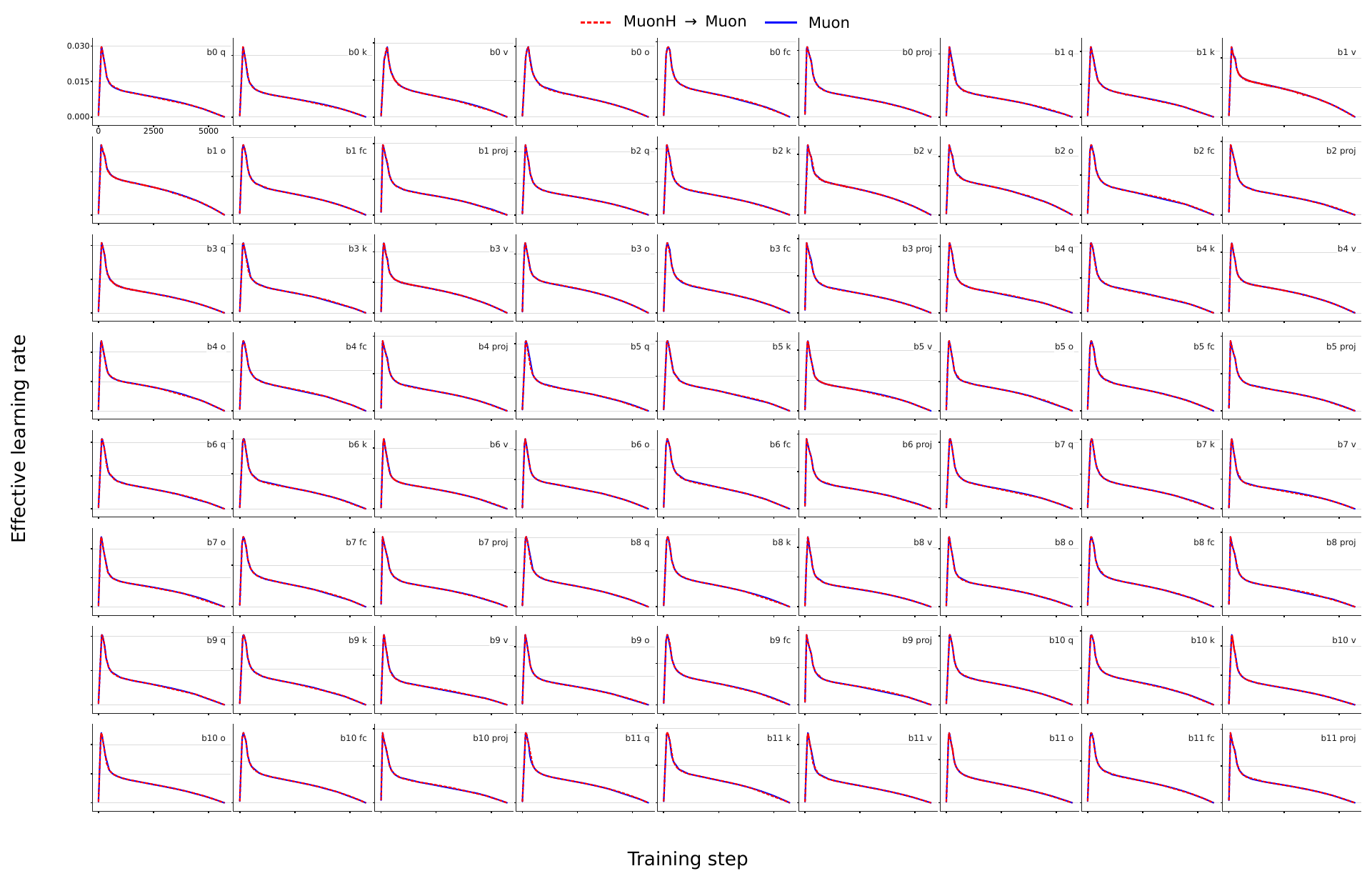}
    \caption{Parameter-wise effective-learning-rate tracking for the non-scale-invariant network under a schedule decaying to zero. Each small panel corresponds to one hidden matrix. The red dashed curves show the HyperTransfer+ estimates, while the blue curves show the corresponding Muon references.}
    \label{fig:data6-lreff0-style}
\end{figure}
\begin{figure}[p]
    \centering
    \includegraphics[width=1\linewidth]{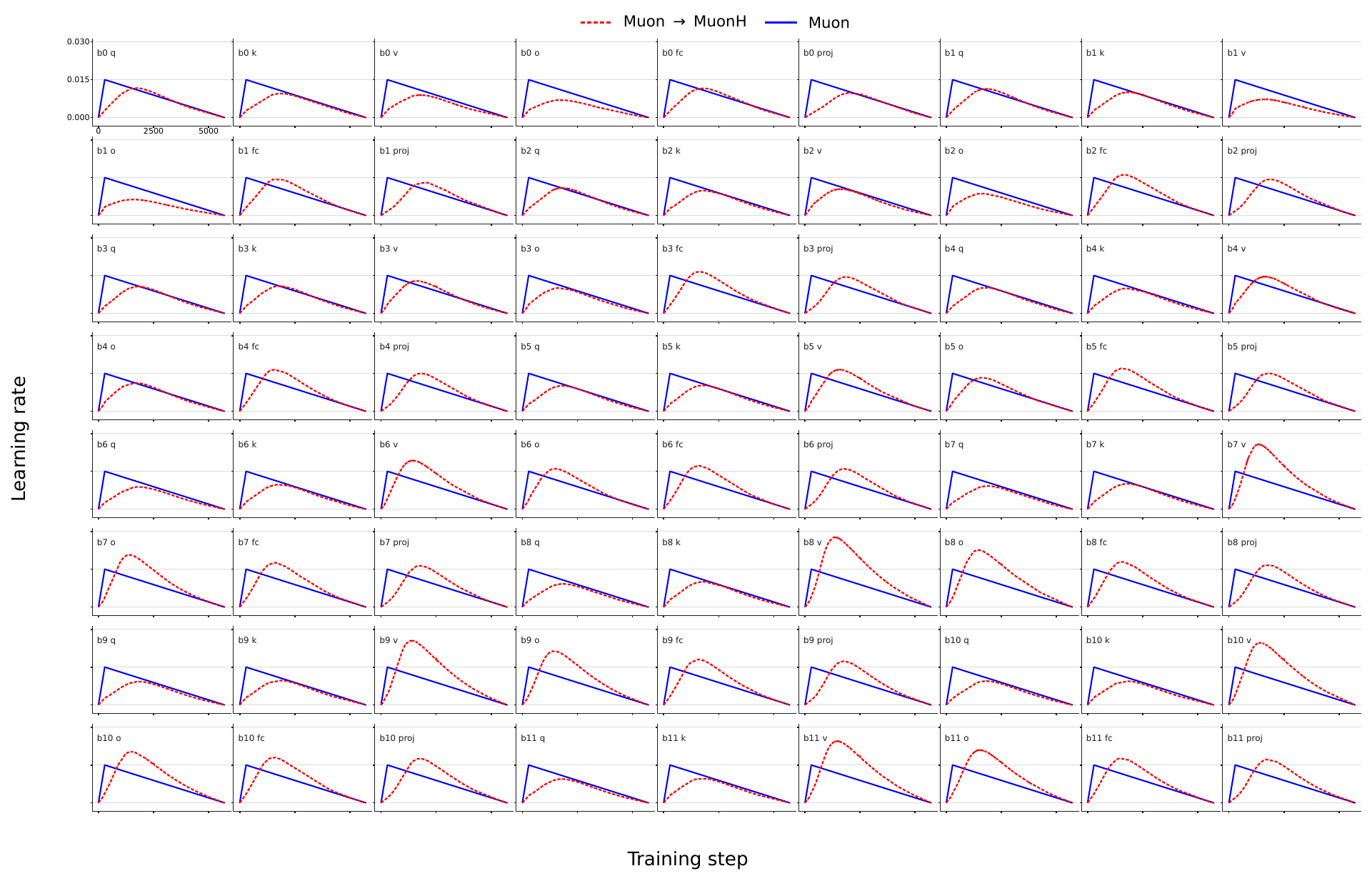}
    \caption{Parameter-wise nominal learning-rate schedules of Muon induced by Inverse HyperTransfer+ on the non-scale-invariant network, using a reference Hyperball schedule that decays to zero. Each panel corresponds to one hidden matrix. The red dashed curve shows the nominal learning rate \(\eta_t\) assigned to Muon, while the blue solid curve shows the nominal learning-rate schedule of the independently tuned Muon reference. }
    \label{fig:data6-Lr0-style}
\end{figure}
\clearpage

\subsubsection{Results with the Learning Rate Decayed to \(0.1\times\) Peak
}
Figures~\ref{fig:data8-norm0-style}--\ref{fig:data8-Lr0-style} present the corresponding parameter-wise results for $\mathcal{M}_2$ under a learning-rate schedule decaying to $0.1\times$ the peak learning rate.

We also observe that, after using Inverse HyperTransfer to align Muon's dynamics with those of MuonH, the resulting Muon schedule assigns different peak learning rates to different matrix parameters, and these schedules all contain a long warmup phase. This suggests that adjusting the learning-rate schedule separately for each parameter block may be a more appropriate choice than using a single global schedule.
\begin{figure}[H]
    \centering
    \includegraphics[width=1\linewidth]{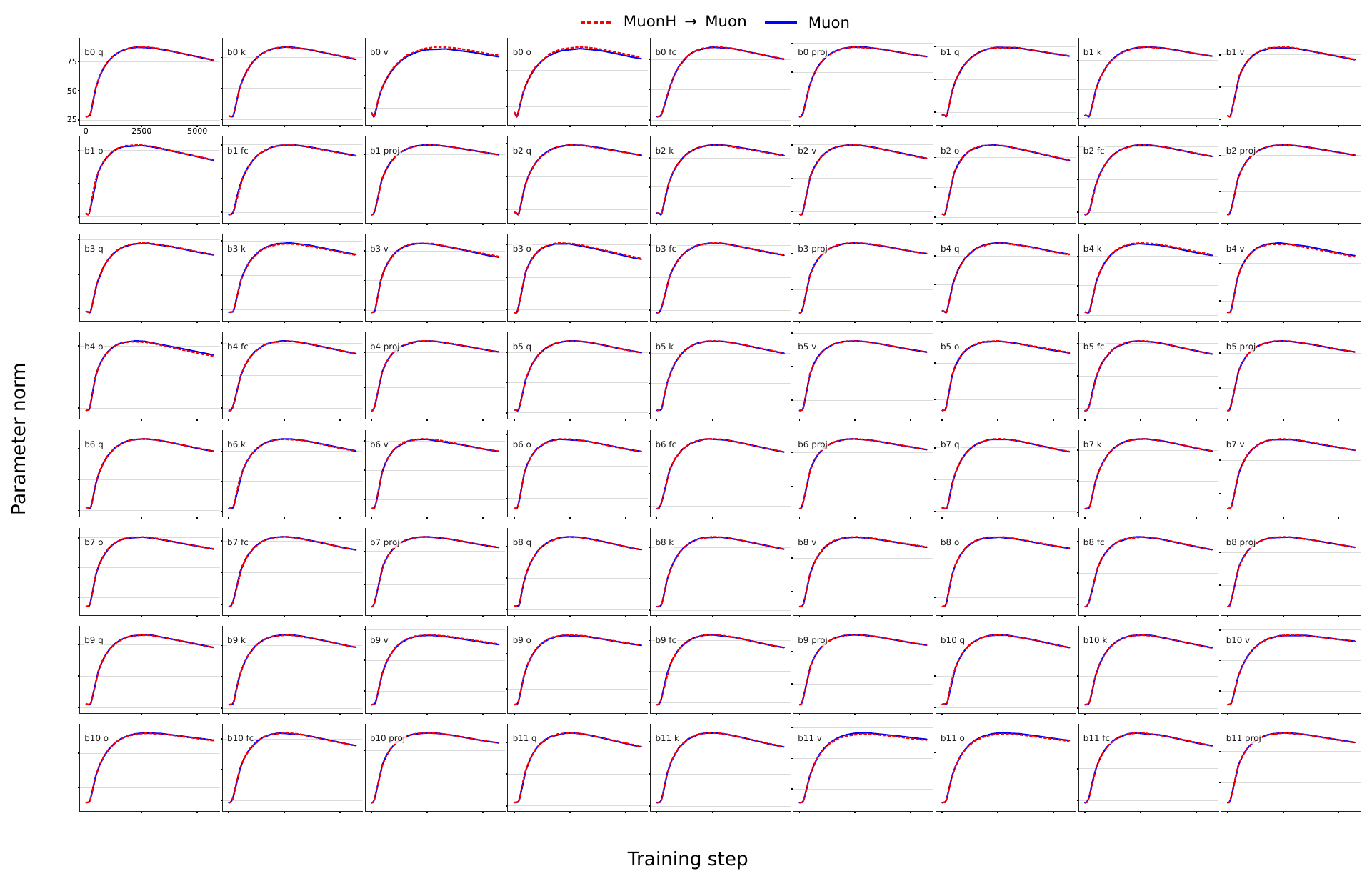}
    \caption{Parameter-wise proxy norm tracking under HyperTransfer+ on the non-scale-invariant network with a learning-rate schedule decaying to \(0.1\times\) the peak learning rate. Each panel corresponds to one hidden matrix. The red dashed curve shows the proxy norm \(s_t\) maintained by the \(\mathrm{MuonH}\rightarrow\mathrm{Muon}\) transfer run, while the blue solid curve shows the actual parameter norm \(\lVert w_t\rVert\) measured along the target Muon trajectory. }
    \label{fig:data8-norm0-style}
\end{figure}
\clearpage
\begin{figure}[p]
    \centering
    \includegraphics[width=1\linewidth]{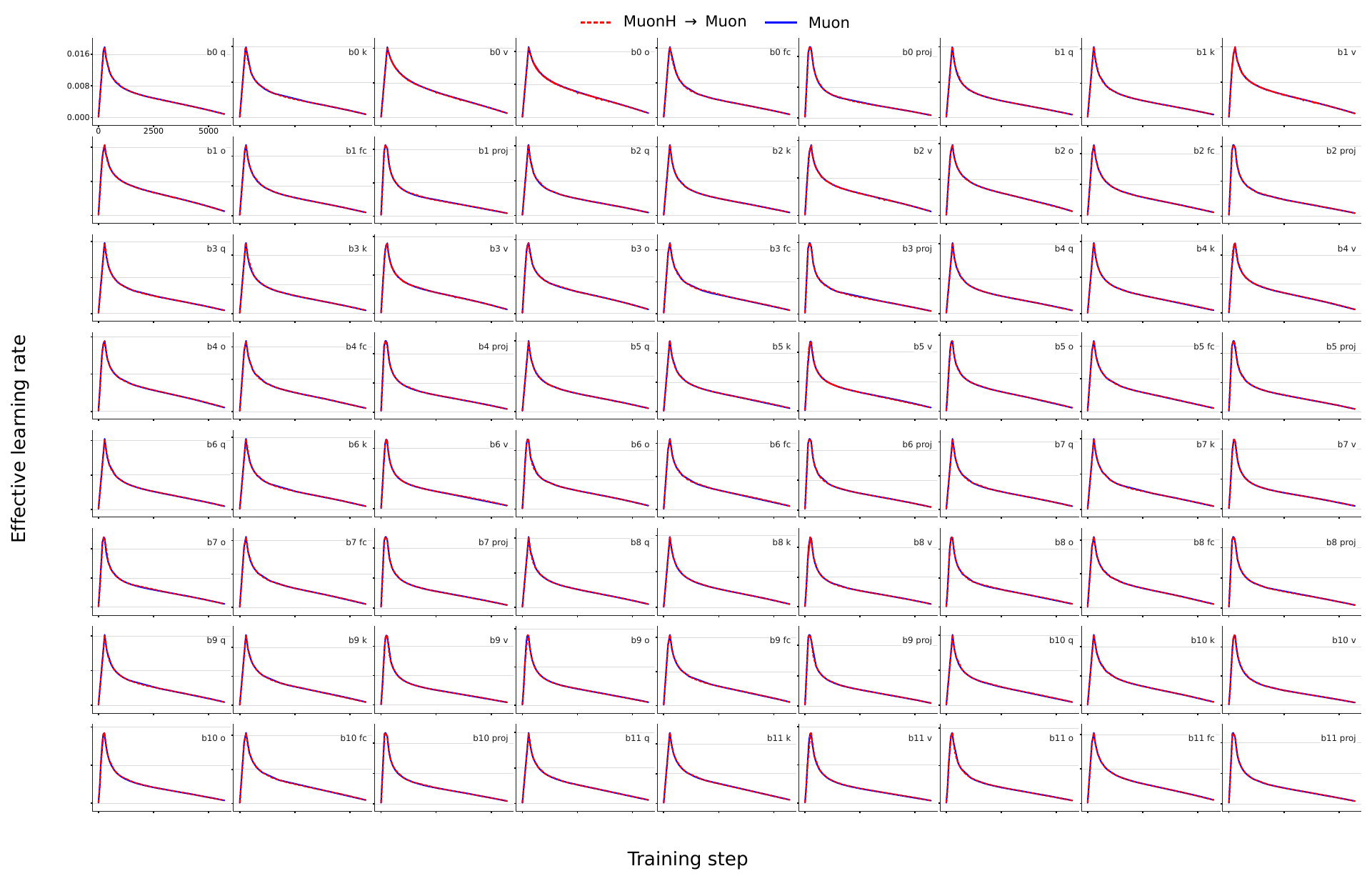}
    \caption{Parameter-wise effective-learning-rate tracking for the non-scale-invariant network under a schedule decaying to \(0.1\times\) the peak learning rate. Each small panel corresponds to one hidden matrix. The red dashed curves show the HyperTransfer+ estimates, while the blue curves show the corresponding Muon references.}
    \label{fig:data8-lreff0-style}
\end{figure}

\begin{figure}[p]
    \centering
    \includegraphics[width=1\linewidth]{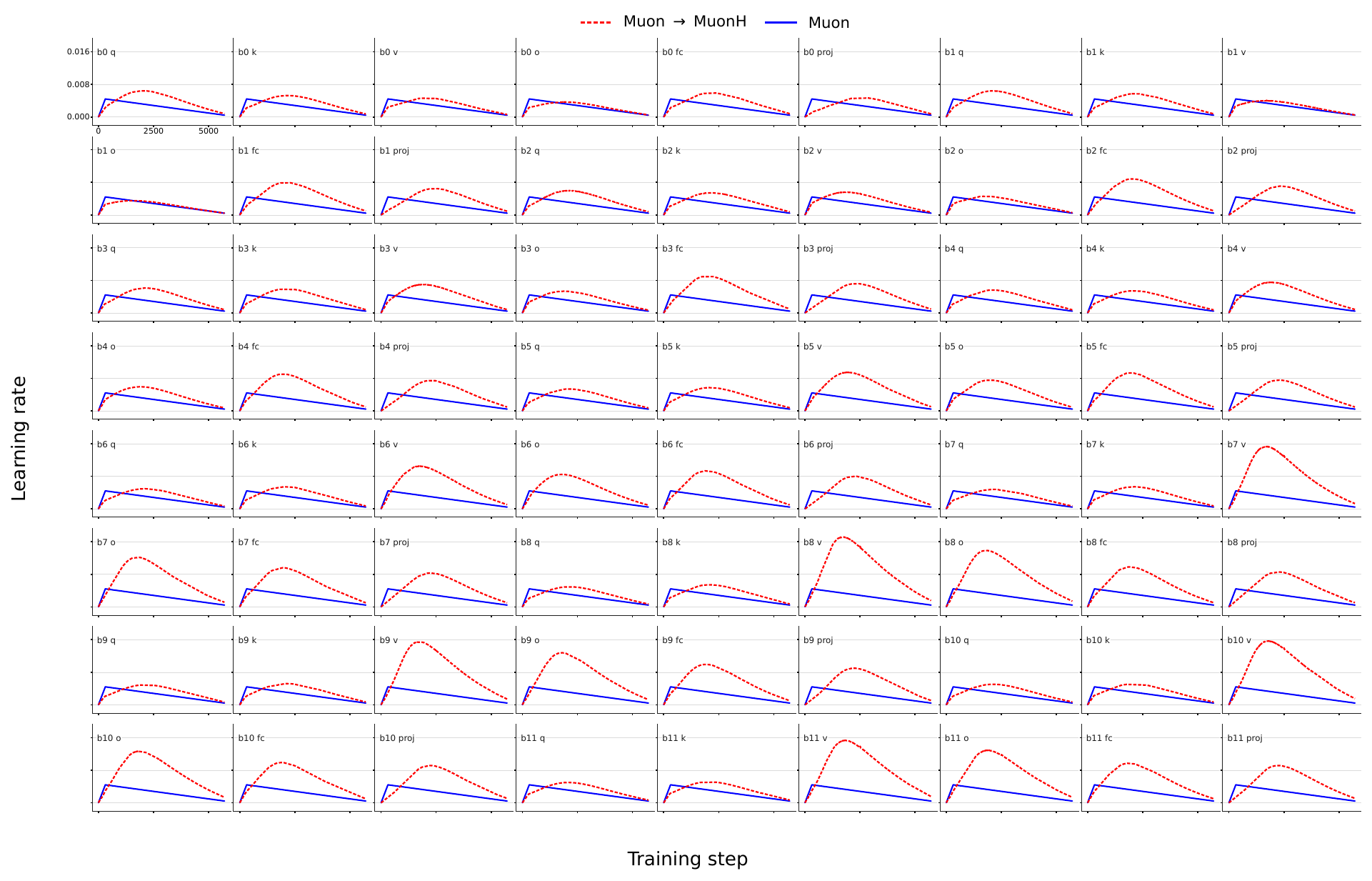}
    \caption{Parameter-wise nominal learning-rate schedules of Muon induced by Inverse HyperTransfer+ on the non-scale-invariant network, using a reference Hyperball schedule that decays to \(0.1\times\) peak learning rate. Each panel corresponds to one hidden matrix. The red dashed curve shows the nominal learning rate \(\eta_t\) assigned to Muon, while the blue solid curve shows the nominal learning-rate schedule of the independently tuned Muon reference. }
    \label{fig:data8-Lr0-style}
\end{figure}

\clearpage
\newpage

\section{Other Experimental Results}
\subsection{HyperTransfer for AdamW}
\label{secapp:hyper_adam}

We further evaluate HyperTransfer with AdamW and its Hyperball counterpart, AdamH, on the scale-invariant network \(\mathcal{M}_1\) under a linear learning-rate schedule decaying to zero. As shown in Figure~\ref{adam}, HyperTransfer enables AdamH to closely follow the target AdamW validation-loss trajectory, while Inverse HyperTransfer enables AdamW to reproduce the target AdamH dynamics. The agreement persists throughout training. 
\begin{figure}[htbp]
    \centering
    \includegraphics[width=1\linewidth]{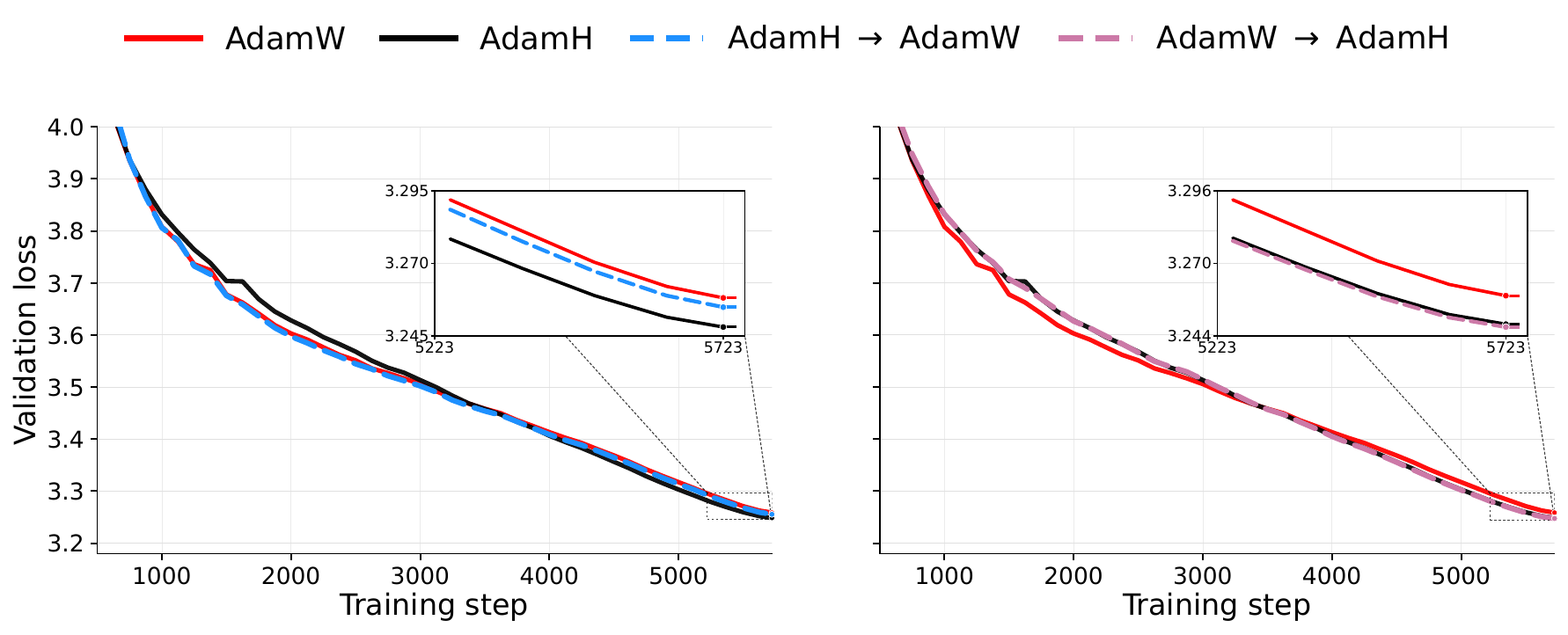}
    \caption{Bidirectional transfer between AdamW and AdamH on the scale-invariant network \(\mathcal{M}_1\) under a linear learning-rate schedule decaying to zero. The left panel shows HyperTransfer, where AdamH \(\rightarrow\) AdamW closely follows the target AdamW loss trajectory. The right panel shows Inverse HyperTransfer, where AdamW \(\rightarrow\) AdamH closely follows the target AdamH trajectory. }
    \label{adam}
\end{figure}
\subsection{HyperTransfer on dense language
model }
We further evaluate HyperTransfer and Inverse HyperTransfer on a
500M dense language model trained on 70B tokens. This experiment
provides an additional large-scale validation of the proposed bidirectional
transfer framework.

Figure~\ref{dense-forward} evaluates the HyperTransfer. The upper panel
shows the training-loss trajectories of the Muon Optimizer and the
transferred MuonH optimizer, while the lower panel shows the difference
between their training losses. The transferred trajectory closely follows the
target trajectory.

Figure~\ref{dense-inverse} evaluates the Inverse HyperTransfer. The upper panel
shows the training-loss trajectories of the target MuonH optimizer and the
Base Optimizer obtained through Inverse HyperTransfer, while the lower panel
shows the difference between their training losses. The close agreement in
both directions demonstrates that HyperTransfer and Inverse HyperTransfer remains effective in a
large-scale dense language-model setting.

\begin{figure}[htbp]
    \centering
    \includegraphics[width=1\linewidth]{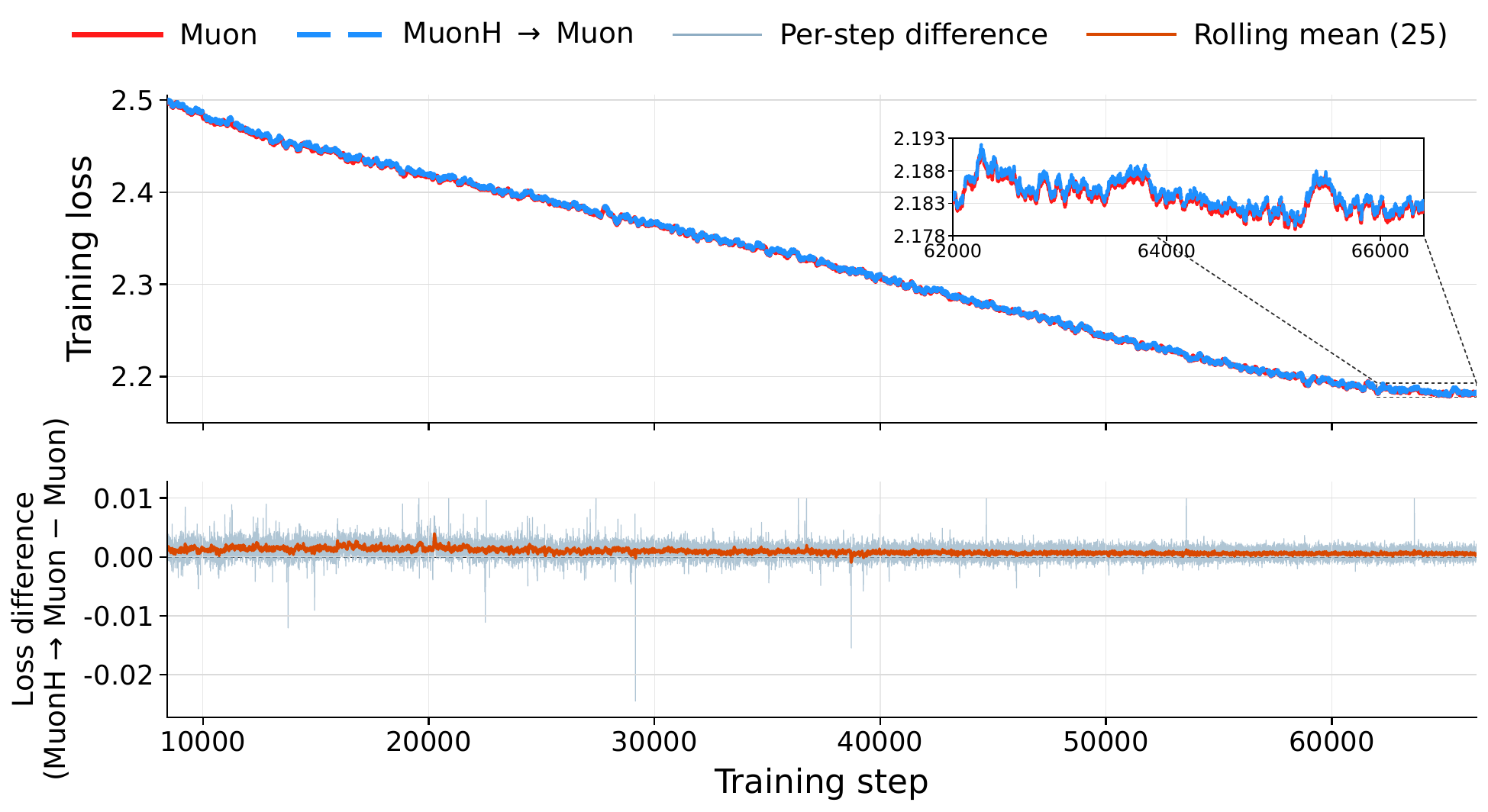}
    \caption{HyperTransfer on the 500M-parameter dense language model
    trained on 70B tokens. The upper panel compares the training-loss
    trajectories of the target Muon Optimizer and the transferred MuonH. The lower panel shows the difference between their training
    losses. The results demonstrate that HyperTransfer maintains its transfer capability even in long-duration training.}
    \label{dense-forward}
\end{figure}

\begin{figure}[htbp]
    \centering
    \includegraphics[width=1\linewidth]{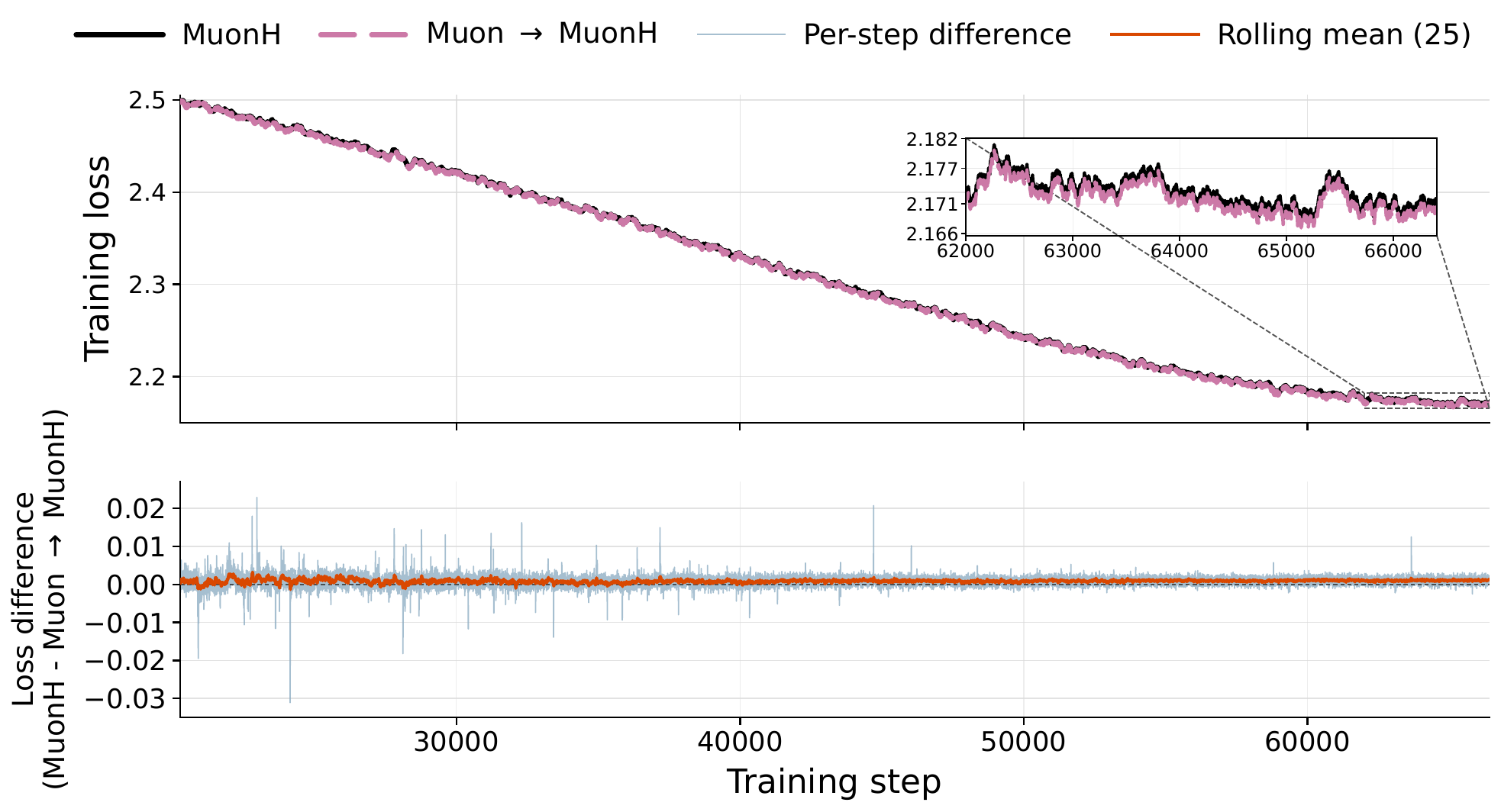}
    \caption{Inverse HyperTransfer on the 500M-parameter dense language model
    trained on 70B tokens. The upper panel compares the training-loss
    trajectories of the target MuonH optimizer and the Base Muon Optimizer
    obtained through Inverse HyperTransfer. The lower panel shows the
    difference between their training losses. The results demonstrate that Inverse HyperTransfer maintains its transfer capability even in long-duration training.}
    \label{dense-inverse}
\end{figure}

\section{Controlling the Effective Learning-Rate Schedule of Base Optimizers by Fair Learning Rate}
\label{app:fair-lr}

The same nominal learning-rate schedule need not produce the same effective
learning-rate schedule for different Base Optimizers. The reason is that the
relative update size depends not only on the nominal learning rate, but also
on the optimizer-dependent update norm and the evolving parameter norm. This
distinction is especially relevant when comparing with Hyperball, whose
fixed parameter norm and normalized update direction make its nominal
learning rate directly determine the relative update size. Our goal is
therefore to control the effective learning-rate schedule of a Base
Optimizer while allowing its parameter norm to evolve freely during
training. To this end, we adjust only the scalar learning rate at each
iteration, without introducing a fixed radius, radial projection, or
prescribed update magnitude. The Base Optimizer uses the update is
\(
    w_{t+1}=(1-\eta_t\lambda)w_t-\eta_tu_t,
\)
and its effective learning rate is defined as:
\begin{equation}
    \eta_t^{\mathrm{eff}}
    =
    \frac{\eta_t\|u_t\|}
    {(1-\eta_t\lambda)\|w_t\|}.
    \label{eq:fair-effective-learning-rate}
\end{equation}

Thus, the effective schedule depends not only on the nominal schedule
$\{\eta_t\}$, but also on the evolving ratio $\|u_t\|/\|w_t\|$. Since this
ratio is determined by the optimizer state and the training trajectory, a
nominal learning-rate schedule alone does not generally specify the
effective-learning-rate schedule in advance. This observation is closely
related to the motivation behind Inverse HyperTransfer, which also uses the
current parameter norm and update norm to transform between nominal and
effective learning rates.

Following the same inversion principle, we define the \textit{Fair Learning
Rate} as an online conversion from a prescribed effective-learning-rate
schedule to the nominal learning rate of a Base Optimizer. Let
$\{\eta_t^{\mathrm{eff}}\}$ denote the target effective-learning-rate
schedule. At each iteration, we set the nominal learning rate supplied to the
Base Optimizer to:

\begin{equation}
    \eta_t^{\mathrm{fair}}
    =
    \frac{\eta_t^{\mathrm{eff}}\|w_t\|}
    {\|u_t\|+\eta_t^{\mathrm{eff}}\lambda\|w_t\|}.
    \label{eq:fair-learning-rate}
\end{equation}
Assuming $\|w_t\|>0$ and $\|u_t\|>0$, substituting
$\eta_t=\eta_t^{\mathrm{fair}}$ into Equation~\eqref{eq:fair-effective-learning-rate}
gives the prescribed $\eta_t^{\mathrm{eff}}$ at every iteration. This is an
algebraic identity at each step and does not depend on the future evolution
of $\|u_t\|/\|w_t\|$. In particular, if the target schedule satisfies
$\eta_T^{\mathrm{eff}}/\max_t\eta_t^{\mathrm{eff}}=r$, then the realized
effective schedule has the same ratio by construction.

This construction controls the effective learning-rate schedule without
introducing the fixed radius $R$ or prescribing the update magnitude used by
Hyperball. The parameter norm is updated by the original Base Optimizer:
\begin{equation}
    \|w_{t+1}\|=\|(1-\eta_t^{\mathrm{fair}}\lambda)w_t-
    \eta_t^{\mathrm{fair}}u_t\|,
\end{equation}
and is therefore free to evolve during
training. Consequently, this construction provides a controlled comparison
with respect to effective learning rate, while making no claim that
optimizers with different update directions or internal states are otherwise
equivalent. Any remaining differences in optimization behavior arise from
the update direction, optimizer state, and norm trajectory rather than from
an uncontrolled effective-learning-rate schedule.

When $\lambda=0$, Equation~\eqref{eq:fair-learning-rate} reduces to
$\eta_t^{\mathrm{fair}}=\eta_t^{\mathrm{eff}}\|w_t\|/\|u_t\|$.
For $u_t$ equal to the gradient, this is the same norm-ratio scaling used by
LARS~\citep{you2017largebatchtrainingconvolutional}.
\section{Understanding Weight Decay via HyperTransfer}
\label{app:weight-decay-hypertransfer}

Hyperball keeps the parameter norm fixed by a radial projection. Although it
does not apply an explicit weight-decay term, this projection introduces a
step-dependent radial rescaling. Let $\bar{w}_{t+1}^H$ denote the Hyperball
iterate before projection. Then:
\begin{equation}
    \bar{w}_{t+1}^H
    =
    w_t^H-\eta_t^H R\frac{u_t^H}{\|u_t^H\|},
    \qquad \|w_t^H\|=R.
    \label{eq:wd-hyperball-preprojection}
\end{equation}

The projected iterate is:
\begin{equation}
    w_{t+1}^H
    =
    \frac{R}{\|\bar{w}_{t+1}^H\|}\bar{w}_{t+1}^H
    =
    \frac{R}{\|\bar{w}_{t+1}^H\|}w_t^H
    -
    \eta_t^H\frac{R^2}{\|\bar{w}_{t+1}^H\|}
    \frac{u_t^H}{\|u_t^H\|}.
    \label{eq:wd-hyperball-form}
\end{equation}

The first term has the same form as the factor
$1-\eta_t\lambda$ in a Base Optimizer. More precisely, the radial factor is
$R/\|\bar{w}_{t+1}^H\|$, which corresponds to the step-dependent weight-decay
coefficient $(1-R/\|\bar{w}_{t+1}^H\|)/\eta_t^H$. Thus, Hyperball can be
viewed as using an implicit weight decay together with a rescaled update.

To compare two Base Optimizers with different weight-decay coefficients,
let their updates be:
\begin{equation}
    w_{t+1}^{(1)}
    =
    (1-\eta_t^{(1)}\lambda_1)w_t^{(1)}
    -\eta_t^{(1)}u_t^{(1)},
    \qquad
    s_t^{(1)}=\|w_t^{(1)}\|.
    \label{eq:wd-base-one}
\end{equation}

The second optimizer uses weight-decay coefficient $\lambda_2$ and updates
according to:
\begin{equation}
    w_{t+1}^{(2)}
    =
    (1-\eta_t^{(2)}\lambda_2)w_t^{(2)}
    -\eta_t^{(2)}u_t^{(2)},
    \qquad
    s_t^{(2)}=\|w_t^{(2)}\|.
    \label{eq:wd-base-two}
\end{equation}

Assume that the two parameters lie on the same ray at iteration $t$. For a
scale-invariant network, their gradients differ only by the scale of the
parameters. Therefore, the gradient entering the second optimizer can be
rescaled before the momentum or adaptive-state recursion as follows:
\begin{equation}
    \widetilde{g}_t^{(2)}
    =
    \frac{s_t^{(2)}}{s_t^{(1)}}g_t^{(2)}
    =
    g_t^{(1)}.
    \label{eq:wd-gradient-accumulation}
\end{equation}

With identical optimizer-state recursions and initial states, this gradient
accumulation makes the two update directions identical, so we write
$u_t^{(1)}=u_t^{(2)}=u_t$. The two optimizers then have the same angular
update when their effective learning rates agree. The required condition is:
\begin{equation}
    \eta_t^{\mathrm{eff},(1)}
    =
    \frac{\eta_t^{(1)}\|u_t\|}
    {(1-\eta_t^{(1)}\lambda_1)s_t^{(1)}}
    =
    \frac{\eta_t^{(2)}\|u_t\|}
    {(1-\eta_t^{(2)}\lambda_2)s_t^{(2)}}
    =
    \eta_t^{\mathrm{eff},(2)}.
    \label{eq:wd-effective-alignment}
\end{equation}

If the first optimizer is taken as the target, the nominal learning rate for
the second optimizer is therefore chosen as:
\begin{equation}
    \eta_t^{(2)}
    =
    \frac{\eta_t^{(1)}s_t^{(2)}}
    {(1-\eta_t^{(1)}\lambda_1)s_t^{(1)}
    +\eta_t^{(1)}\lambda_2s_t^{(2)}}.
    \label{eq:wd-transferred-lr}
\end{equation}

The proxy norms must be updated separately because the two optimizers use
different weight-decay coefficients. For the optimizer with coefficient
$\lambda_1$, the update is:
\begin{equation}
    \begin{aligned}
    \left(s_{t+1}^{(1)}\right)^2
    ={}
    \left(1-\eta_t^{(1)}\lambda_1\right)^2
    \left(s_t^{(1)}\right)^2
    +\left(\eta_t^{(1)}\right)^2\|u_t\|^2 
    -2\eta_t^{(1)}
    \left(1-\eta_t^{(1)}\lambda_1\right)
    \frac{s_t^{(1)}}{s_t^{(2)}}\left\langle u_t,w_t^{(2)}\right\rangle .
    \end{aligned}
\end{equation}

For the optimizer with coefficient $\lambda_2$, it is:
\begin{equation}
    \begin{aligned}
    \left(s_{t+1}^{(2)}\right)^2
    ={}
    \left(1-\eta_t^{(2)}\lambda_2\right)^2
    \left(s_t^{(2)}\right)^2
    +\left(\eta_t^{(2)}\right)^2\|u_t\|^2 -2\eta_t^{(2)}
    \left(1-\eta_t^{(2)}\lambda_2\right)
    \frac{s_t^{(2)}}{s_t^{(1)}}\left\langle u_t,w_t^{(1)}\right\rangle .
    \end{aligned}
\end{equation}

Each recursion is exactly the squared norm of the Base
Optimizer update. Hence, if the two optimizers start on the same ray, the
gradient accumulations synchronize their update directions, and
Equation~\eqref{eq:wd-effective-alignment} matches their effective angular
learning rates, the two parameter trajectories remain on the same ray. On a
scale-invariant network, this implies:\(
    f\left(w_t^{(1)}\right)
    =
    f\left(w_t^{(2)}\right)
\) for every $t$.

Therefore, two optimizers with different weight-decay coefficients can in
principle follow the same loss trajectory when their gradient accumulations
and effective learning rates are aligned. For a non-scale-invariant network,
the same construction requires evaluating the gradients at the corresponding
scale representatives, and loss equality is not implied without an additional
assumption.

\section{Limitations and Future Work}

Several limitations remain. For non-scale-invariant networks, HyperTransfer+ and Inverse HyperTransfer+ do not generally imply \(f(w_t)=f(w_t^H)\), because the fixed-radius Hyperball parameter is evaluated at a different scale from the Base parameter. For HyperTransfer+, the maintained relation \(w_t=(s_t/R)w_t^H\) gives \(f(w_t)=f((s_t/R)w_t^H)\). The theoretical correspondence also requires \(w_t\) and \(w_t^H\) to remain exactly aligned in direction. In finite-precision computation, accumulated numerical errors can introduce a small angular deviation and a corresponding discrepancy between the loss values. Finally, the proxy norm \(s_t\) is currently determined by the fixed analytical update in Equation~\eqref{up}. One direction for future work is to make \(s_t\) learnable. More generally, the proxy-norm schedule and the nominal learning-rate schedule could be decoupled and adjusted jointly, providing two independently tunable controls for shaping a nonlinear effective-learning-rate schedule.
\end{document}